\documentclass[journal]{IEEEtran}
\usepackage[backend=biber,
sorting=none,
doi=false,
isbn=false,
url=true,
maxcitenames=2,
mincitenames=1,
style=ieee,]{biblatex}
\ifCLASSINFOpdf
\else
\fi
\usepackage{amsmath}
\usepackage{amsthm}
\usepackage{bm}
\usepackage{booktabs}
\usepackage{multirow}
\newtheorem{prop}{Proposition}

\allowdisplaybreaks

\usepackage{algpseudocode}
\usepackage{algorithm}

\usepackage{xcolor}
\allowdisplaybreaks
\makeatletter
\DeclareRobustCommand{\change}{%
  \@bsphack
  \leavevmode
  \color{blue}%
  \@esphack
}
\DeclareRobustCommand{\stopchange}{%
  \@bsphack
  \normalcolor
  \@esphack
}
\makeatother

\usepackage{subfigure}
\usepackage{graphicx}
\usepackage[caption=false]{subfig}
\usepackage{amsfonts}
\usepackage{color,soul}

\usepackage{xcolor}

\begin{document}
%
\title{Traffic signal prediction on transportation networks using spatio-temporal correlations on graphs}
%
%
%

\author{Semin Kwak,~
        Nikolas Geroliminis,~
        and~Pascal Frossard
}

\maketitle
\begin{abstract}
Multivariate time series forecasting poses challenges as the variables are intertwined in time and space, like in the case of traffic signals. Defining signals on graphs relaxes such complexities by representing the evolution of signals over a space using relevant graph kernels such as the heat diffusion kernel. However, this kernel alone does not fully capture the actual dynamics of the data as it only relies on the graph structure. The gap can be filled by combining the graph kernel representation with data-driven models that utilize historical data. This paper proposes a traffic propagation model that merges multiple heat diffusion kernels into a data-driven prediction model to forecast traffic signals. We optimize the model parameters using Bayesian inference to minimize the prediction errors and, consequently, determine the mixing ratio of the two approaches. Such mixing ratio strongly depends on training data size and data anomalies, which typically correspond to the peak hours for traffic data. The proposed model demonstrates prediction accuracy comparable to that of the state-of-the-art deep neural networks with lower computational effort. It notably achieves excellent performance for long-term prediction through the inheritance of periodicity modeling in data-driven models.

\end{abstract}


%
\IEEEpeerreviewmaketitle

\section{Introduction}
%
%
%
%
\IEEEPARstart{M}{ultivariate} time-series prediction is an important task since many real-life problems can be modeled within this framework, such as weather forecasting~\cite{wan2019multivariate, das2017sembnet, ouyang2017combined}, traffic prediction~\cite{wang2019multiple, huang2019dsanet, chandra2009predictions,mai2012multivariate,kwak2020travel,cavalcante2017lasso, nicholson2020high, li2018diffusion, cui2019traffic, chen2019gated, zhang2019spatial, yu2018spatio, zhao2019t,  wu2019graph, wu2020connecting, xu2018graph}, power consumption forecasting~\cite{du2020multivariate, nicholson2020high}, and others~\cite{yao2018deep, wu2020connecting, huang2019dsanet, munkhdalai2019end, du2003univariate}.
In transportation sensor networks, output signals from neighboring sensors may be similar or vastly different, as shown in Fig.~\ref{fig:sensornetwork}(a) and (b). Therefore, in this example, sensor A's signal can be utilized to predict sensor B's as the two signals are well correlated. However, the signal of sensor C is not correlated with that of sensor B, so it may not contribute to the prediction; 
Sensor C is located after an intersection, and most traffic demands flow in another direction in the intersection, therefore, the sensor rarely suffers congestion.
Naturally freeway congestion (expressed with a sharp decrease in the average speed of vehicles) is initiated at a bottleneck location such as an on-ramp merging area with high entrance flow or an incident location. Then, it propagates backwards with a finite speed, which is 3 to 4 times smaller than the speed of traffic. Fig.~\ref{fig:sensornetwork}(c) shows an example of congestion propagation in I-280 and I-880 freeways in California. Note that there is a drastic decrease in the speed at a location (sensor B) and a time (around 3 pm) that propagates through the traffic stream (this is called a shockwave). Once demand for travel decreases congestion disappears by following the opposite trend during the offset of congestion with a forward moving wave. Note that this propagation speed is not constant and depends on the concentration or density of vehicles (with units of veh/km) on the two sides of the shockwave. There are various theories in transportation science to describe the mechanisms of stop-and-go phenomena inspired by fluid and heat diffusion models (see \cite{helbing2001traffic} for an overview). 

Due to complex spatio-temporal correlation, the choice of model greatly influences the predictive performance.
For small-scale sensor networks, such correlations can be estimated directly from historical data~\cite{chandra2009predictions,mai2012multivariate,kwak2020travel,cavalcante2017lasso, nicholson2020high}.
The vector Auto Regression (AR) is a representative model for multivariate time series forecasting~\cite{kwak2020travel,cavalcante2017lasso, nicholson2020high}. 
In this model, regression parameters, or correlations between sensors, are estimated solely using historical data. In our previous work~\cite{kwak2020travel}, we implemented a predictor that explicitly expresses the periodicity of traffic signals with temporally localized vector AR model.
However, these data-driven models are not suitable for multivariate time series prediction with a large number of variables because the number of correlations to be estimated increases exponentially compared with the number of sensors, which causes incompleteness of the estimator (or overfitting).

Recently, many studies have prioritized the correlations among sensors by defining signals on graphs~\cite{li2018diffusion,yu2018spatio,cui2019traffic,zhao2019t,chen2019gated,zhang2019spatial, xu2018graph, wu2019graph, wu2020connecting}. 
In particular, in transportation networks, the physical travel distance between sensors is a critical {\it{a priori}} information, the closer the sensors are in space, the higher the correlation~\cite{battaglia2018relational}.
Utilizing this information, the authors had extracted the signal's spatial features through the heat propagation kernel (or convolutional filter) and passed it to temporal blocks for forecasting, such as recurrent neural network (RNN)~\cite{li2018diffusion, cui2019traffic, chen2019gated, zhang2019spatial} and temporal convolutional layer (TCN)~\cite{yu2018spatio, zhao2019t, wu2019graph, wu2020connecting}.
By introducing this prior information to complex deep neural networks, they  achieved state-of-the-art performance in traffic prediction.

\begin{figure}[!t]
   \centering
    \subfigure[Sensor locations of PEMS-BAY network. The distance between two consecutive sensors in a freeway is 0.6 mile in average. ]{{\includegraphics[width=0.9\columnwidth]{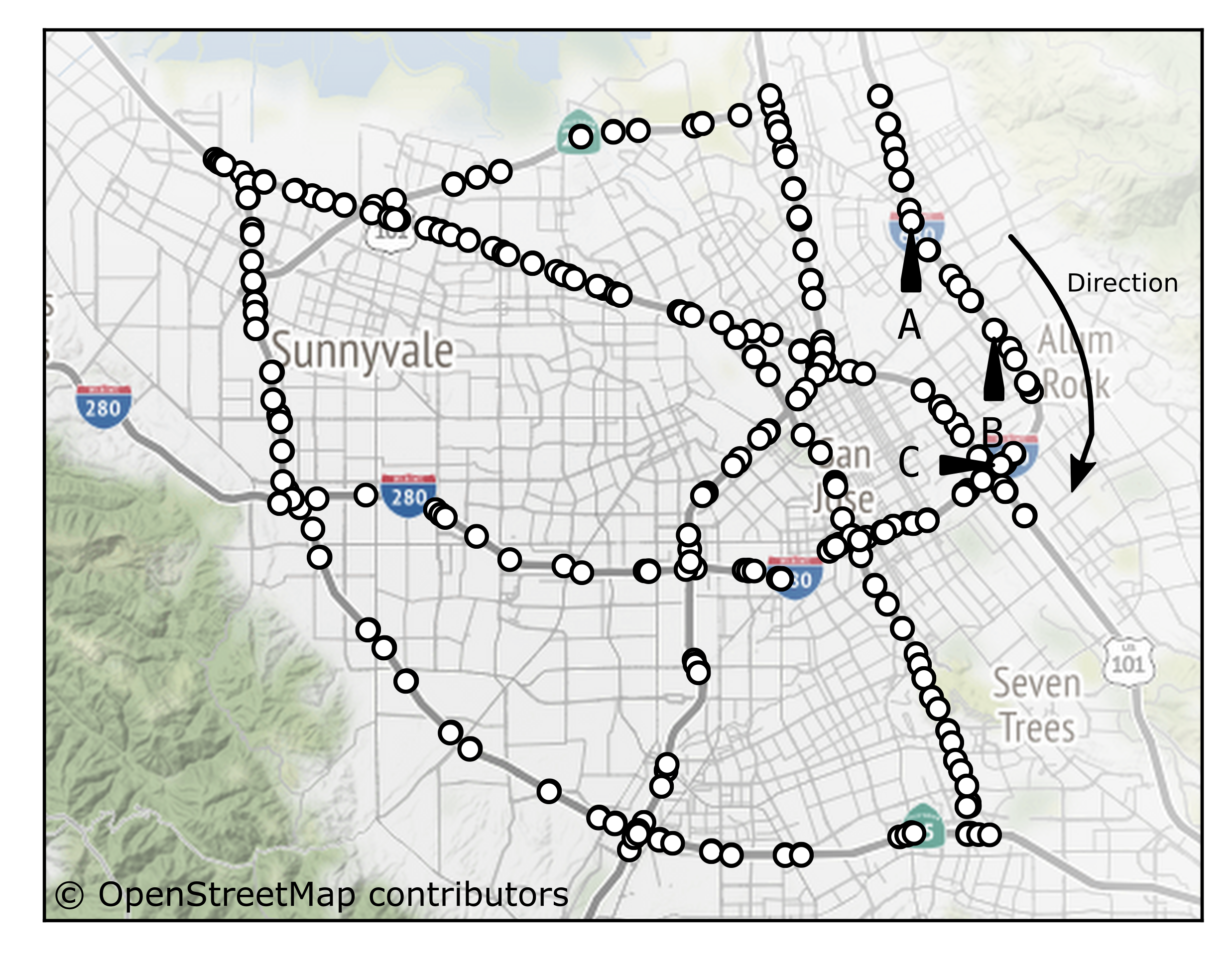}}}
   \subfigure[Signals on different sensors]{{\includegraphics[width=0.9\columnwidth]{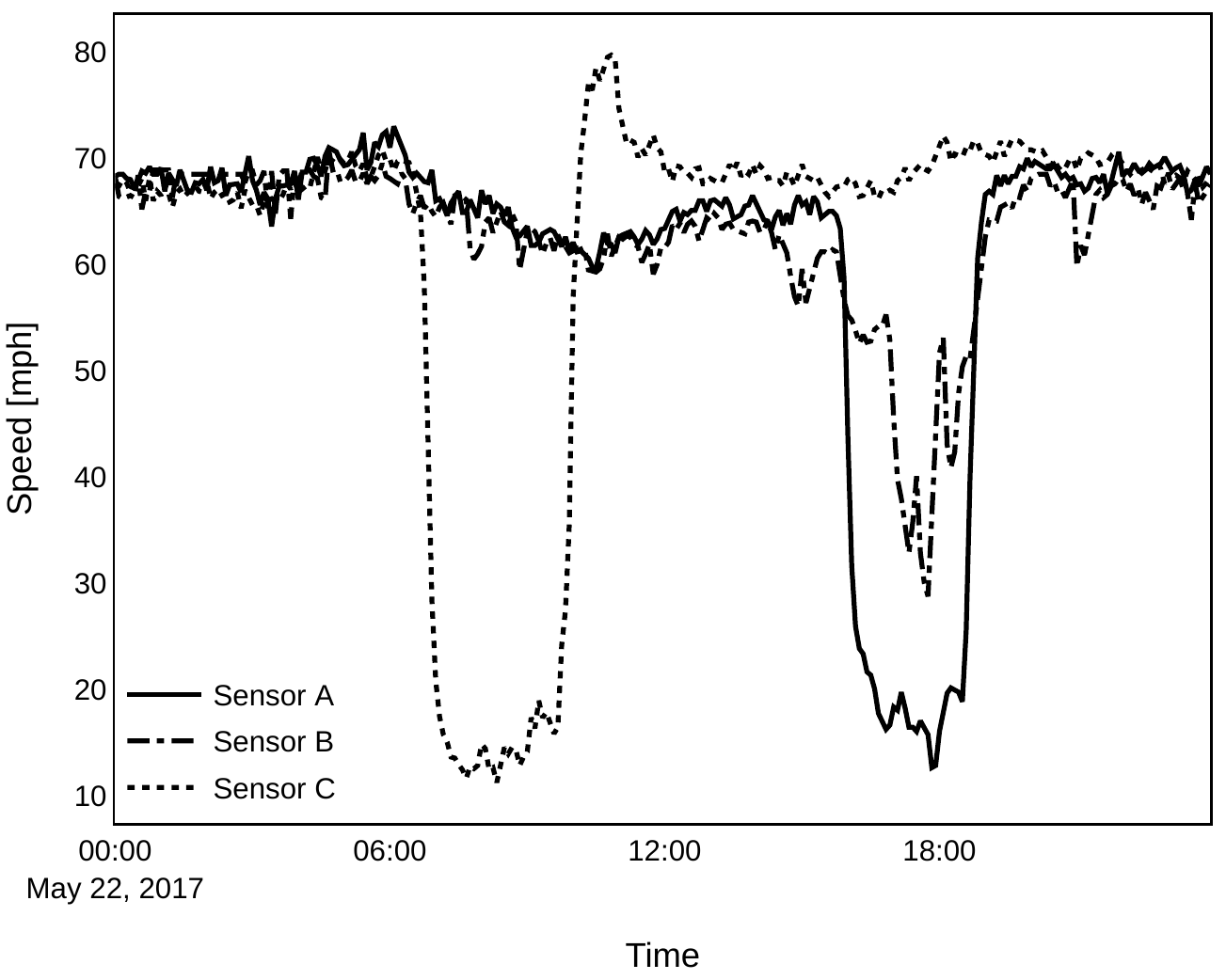}}}
   \subfigure[Speed profile for the evening peak over time and space. The day 2017-05-22 (Monday) is selected. The red dashed lines represent the waves that congestion propagates.]{{\includegraphics[width=\columnwidth]{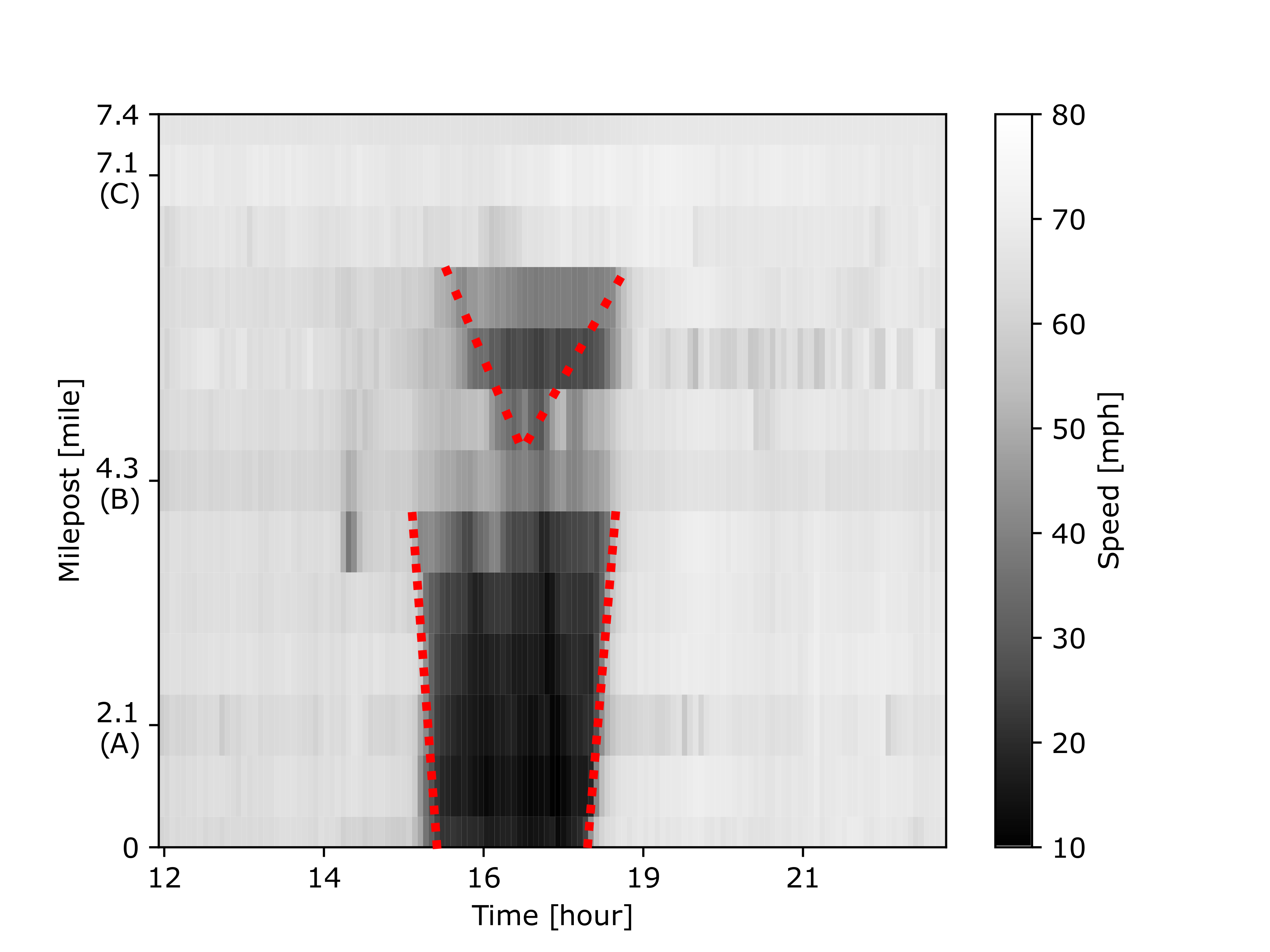}}}
   \caption{A transportation sensor network in California and signals of three different sensors on the network.
   Although the sensors B and C are close to each other in distance, two trafﬁc signals from these sensors show very different patterns.
  }
   \label{fig:sensornetwork}
\end{figure}

However, the two predictors (with and without graphs) each have their own drawbacks.
In the former case, to the best of our knowledge, all studies, which currently show the best performance, construct predictors based on deep neural networks. Therefore, these models require expensive tuning processes of many hyperparameters and relatively long training due to numerical optimization processes.
In the latter case, on the other hand, it can be inefficient concerning the prediction accuracy, especially for large networks when the structural information becomes important.

This paper proposes a new model that combines the advantages of different frameworks by implanting the sensors' structural information into the existing data-driven model~\cite{kwak2020travel}, inheriting the periodicity modeling for the traffic signal.
In most studies, the periodicity of the traffic signal is taken as the input feature of the predictor, such as an encoded vector that represents the time of the day or the day of the week, but the study~\cite{kwak2020travel} instead induces the periodicity of the signal more clearly by making the model itself different for each time.
Each model has a matrix, which should be estimated by historical data, representing the correlation between signals at two consecutive time intervals. As the size of the network is proportional to the size of the matrix, a larger network can lead to overfitting.
In this paper, we resolve the overfitting problem by approximating this matrix to the one derived from data-independent graph topological information, therefore, we estimate only the remainder by data.
In detail, we transform the graph topological information into heat diffusion kernels, which is introduced in~\cite{kondor2002diffusion}, and approximate the matrix to a combination of the heat diffusion kernels.
In the process, we introduce some hyper-parameters. For example, one determines which of the prior or historical datasets is more reliable. Most of the existing studies estimate hyper-parameters through exhaustive search as a cross-validation method using a validation set, but we estimate hyper-parameters directly from data by utilizing Bayesian inference~\cite{mackay1992bayesian}.
As a result, the estimation process is relatively fast as most parameter estimation is performed by analytic calculations except a few ones requiring a numerical optimization process.
Besides, our model is strongly interpretable. 
For example, through the hyper-parameter, it can be seen that during the peak period, traffic prediction is relatively more dependent on data than structural information compared to the non-peak period.
Also, most importantly, predictors based on this model showed comparable performance with a much shorter learning time than state-of-the-art models.
Especially, the proposed model shows great long-term prediction performance as the model captures well the periodicity of traffic signals.
Since the proposed model requires a minimal number of hyper-parameter tuning, it might be applied to other daily periodic graph signal prediction problems easily (e.g., weather forecasting, daily energy consumption prediction).
Here we summarize contributions of the work:
\begin{itemize}
    \item We propose a novel traffic prediction method that successfully integrate graph structural information to the existing data-driven model~\cite{kwak2020travel}. 
    Hyper-parameters are learned directly from data through Bayesian inference rather than by exhaustive search.
    \item Therefore, the training time required for inference is minimal. The trained model is straightforward to analyze, unlike other deep neural network-based models.
    \item It shows prediction performance comparable with deep learning methods especially for long-term prediction.
\end{itemize}

\begin{table}[t!]
\caption{The notations and definitions used in this article.}\label{table:definitions}
\renewcommand{\arraystretch}{1.3}
\begin{tabular}{p{1.5cm} p{6.5cm} }
\toprule
$\mathcal{R}^m$ & $m$-dimensional Euclidean space\\ 
$a,{\bf{a}}, {\bf{A}}$ & Scalar, vector, matrix\\
$\text{diag}({\bf{a}})$ & The diagonal matrix whose diagonal elements are from the vector ${\bf{a}}$\\
$\text{diag}({\bf{A}})$ & The vector whose elements are the diagonal components of the matrix ${\bf{A}}$\\
${\bf{I}}$ & Identity matrix\\
${\bf{1}}$ & All one vector\\
$e^{{\bf{A}}}$ & $\lim_{n\rightarrow\infty}\left({\bf{I}}+\frac{1}{n}{\bf{A}}\right)^n=\sum_{n=0}^\infty \frac{1}{n!}{\bf{A}}^n$\\
$[{\bf{A}}]_{i,j}$ & The element of $i$-th row and $j$-th column of the matrix ${\bf{A}}$\\
$[{\bf{A}}]_{i,:}$ & The slice of $i$-th row of the matrix ${\bf{A}}$ \\
$|{\bf{A}}|$ & The determinant of the matrix ${\bf{A}}$\\
$|\mathcal{S}|$ & The cardinality of the set $\mathcal{S}$\\
$\mathcal{N}(\mu,\sigma^2)$ & A Gaussian distribution which has the probability density function $f(x)=$ $\frac{1}{\sigma\sqrt{2\pi}}\exp\left({-\frac{1}{2}\left(\frac{x-\mu}{\sigma}\right)^2}\right)$\\
$\mathcal{N}({\boldsymbol{\mu}},{\boldsymbol{\Sigma}})$ & A multivariate Gaussian distribution which has the probability density function $f({\bf{x}})=\frac{1}{\sqrt{(2\pi)^N|{\boldsymbol{\Sigma}}|}}\exp\left({-\frac{1}{2}({\bf{x}}-{\boldsymbol{\mu}})^T{\boldsymbol{\Sigma}}^{-1}({\bf{x}}-{\boldsymbol{\mu}})}\right)$\\
$\mathcal{N}({\bf{M}},\sigma^2)$ & $\prod_{i,j}\mathcal{N}([{\bf{M}}]_{i,j},\sigma^2)$\\
$\mathcal{N}({\bf{M}},{\boldsymbol{\Sigma}})$ & $\prod_i \mathcal{N}([{\bf{M}}]_{i,:},{\boldsymbol{\Sigma}})$\\
\bottomrule
\end{tabular}
\end{table}

\section{Data model}
In this section, we describe a mathematical model that represents a relationship between traffic signals that are different in time. 
First, we define traffic signals on a graph and introduce an existing prediction model~\cite{kwak2020travel} using this signals. 
Then, we suggest a model extending the previous one that is applicable for large scale networks by exploiting graph information.
\subsection{Graph signal}
We start with modeling a transportation network using a graph. 
We define an undirected graph $\mathcal{G}=(\mathcal{V},\mathcal{E})$; $\mathcal{V}$ is a set of nodes where each $v\in\mathcal{V}$ denotes a node (sensor) on the graph; $\mathcal{E}$ is a set of edges where each of the edges connects two nodes. 
We define a signal on the nodes of the graph with a traffic feature, in this paper, for instance, speed, which is expressed as a vector ${\bf{x}}_t^d\in\mathcal{R}^{N}$ of a day~$d$ and time~$t$, where the constant $N$ is the number of nodes. 
Therefore, the vector ${\bf{x}}_t^d$ represents a snapshot of speeds at a particular time and day. 
Especially, we express the day index on the vector representation to exploit the periodicity of traffic signals later.

\subsection{Dynamic linear model (DLM)}
In our previous study~\cite{kwak2020travel}, we defined a state equation of traffic in a small-scale transportation network (a path graph) as temporally localized linear models as follows:
\begin{equation}\label{eqn:dlm}
    {\bf{x}}_{t+1}^d={\bf{H}}_t{\bf{x}}_t^d+{\bf{n}}_{t}^d,\forall t\in\left[0,T-1\right].
\end{equation}
We called this model the Dynamic linear model (DLM).
The first time index ($t=0$) corresponds to the beginning of a day (midnight in our work), and the last index ($t=T-1$) refers to the end of the day.
Each entry of the noise vector ${\bf{n}}_t^d\in\mathcal{R}^{N}$ is assumed to be an independent and identically distributed (i.i.d.) random variable, which follows a Gaussian distribution $\mathcal{N}(0,\alpha_t^{-1})$. Here the precision parameter $\alpha_t$ explains how precisely a data pair $({\bf{x}}_t^d, {\bf{x}}_{t+1}^d)$ fits to the model. 
The transition matrix ${\bf{H}}_t$ represents the linear relationship between traffic signals ${\bf{x}}_t^d$ and ${\bf{x}}_{t+1}^d$.

The most important motivation behind this model is that the propagation of traffic features over time occurs periodically on a daily basis.
Consequently, we modeled that the transition matrix ${\bf{H}}_t$ as a time-variant matrix that contains temporally localized (only between two consecutive traffic features) spatio-temporal correlations of every sensor pair regardless of the day of the week, noting that the transition matrix does not have the day index. In other words, we assumed the correlations are identical both for weekends and weekdays~\cite{kwak2020travel}.

In the work~\cite{kwak2020travel}, the transition matrix is estimated by maximizing the likelihood (note that we ignore some parameters such as the regularization parameter and the forgetting factor introduced in the work for the brevity) as follows:
\begin{equation}\label{eqn:likelihood_only}
\begin{aligned}
    \bar{\bf{H}}_t=\underset{{\bf{H}}_t}{\text{argmax }}f({\bf{X}}_{t+1}|{\bf{X}}_{t},{\bf{H}}_{t},\alpha_t)={\bf{X}}_{t+1}{\bf{X}}_t^T({\bf{X}}_{t}{\bf{X}}_t^T)^{-1},
\end{aligned}
\end{equation}
where the collection of the $m$-past signals ${\bf{X}}_t=\begin{pmatrix}{\bf{x}}_t^{0}&{\bf{x}}_t^{1}&\cdots&{\bf{x}}_t^{{m-1}}\end{pmatrix}$. Therefore, the optimal transition matrix is solely determined by the historical data ${\bf{X}}_t$ and ${\bf{X}}_{t+1}$. From Eq.~(\ref{eqn:likelihood_only}) we see that the matrix ${\bf{X}}_{t}{\bf{X}}_t^T$ can be an ill-conditioned matrix when $N$ is large. In other words, the transition matrix ${\bar{\bf{H}}}_t$ can be overfitted by data. In the following subsection, we suggest a method to avoid this problem by utilizing graph topological information.

\subsection{DLM with graph topological information}
In this subsection, we suggest a way to avoid the overfitting problem approximating the transition matrix to a heat diffusion matrix. To achieve this goal, we first define a weight matrix that contains all edge weights between node $v_i$ and $v_j$ using a Gaussian kernel weighting function with a threshold constant $\kappa$:
\begin{equation}\label{eqn:gaussianthreshold}
    [{\bf{W}}]_{i,j}=
    \begin{cases}
      e^{{-\frac{ \text{dist}^2(i,j)}{\sigma^2}}}, & \text{if}\ \text{dist}(i,j)\le\kappa \\
      0, & \text{otherwise}.
    \end{cases}
\end{equation}
The function $\text{dist}(i,j)$ denotes the shortest travel distance on $\mathcal{G}$ between the node $v_i$ and $v_j$:
\begin{equation}
    \text{dist}(i,j)=\min\{\text{dist}(v_i\rightarrow v_j),\text{dist}(v_j\rightarrow v_i)\},
\end{equation}
where the function $\text{dist}(v_i\rightarrow v_j)$ represents the shortest travel distance from node $v_i$ to node $v_j$.
As the graph $\mathcal{G}$ is undirected, the weight matrix is a symmetric matrix, i.e., ${\bf{W}}^T={\bf{W}}$.

The constants $\sigma$ and $\kappa$ are the kernel width and the distance threshold.
If the kernel width is large, the correlation of a pair of nodes is strong (close to one) even though the shortest travel distance between the two nodes is large. 
On the other hand, the smaller the threshold is, the sparser the weight matrix is.

The graph heat diffusion model~\cite{kondor2002diffusion} explains how each vertex propagates its heat to its neighbors on the graph over time.
As congestion evolves from one location to its neighbor over time, we can express the change of traffic features by the heat diffusion model, especially for short-term traffic changes since the total traffic volume of a network is well preserved for the short-term in general.

The kernel on graphs that supports the heat diffusion model is introduced by~\cite{kondor2002diffusion}:
\begin{equation}\label{eqn:heat_diffusion_model}
    {\bf{H}}^{\mathcal{G}}(\tau)=e^{-\tau{\bf{L(\mathcal{G})}}},
\end{equation}
where the constant $\tau$ denotes the diffusion period and the matrix ${\bf{L}}(\mathcal{G})$ is the Laplacian of a graph $\mathcal{G}$. The matrix is defined as
\begin{equation}\label{eqn:laplacian}
    {\bf{L}}(\mathcal{G})=\text{diag}({\bf{W}}{\bf{1}})-{\bf{W}}.
\end{equation}

By definition, two extreme heat diffusion kernels of a connected graph $\mathcal{G}$ are:
\begin{equation}\label{eqn:extreme_prior}
    {\bf{H}}^{\mathcal{G}}(\tau)=
    \begin{cases}
      {\bf{I}}, & \text{when}\ \tau\rightarrow0, \\
      \frac{1}{N}{\bf{11}}^T, & \text{when}\ \tau\rightarrow\infty,
    \end{cases}
\end{equation}
where ${\bf{1}}$ is the vector whose elements are all one.

Therefore, with the heat diffusion kernel, we can describe the diffusion of a traffic signal through the graph $\mathcal{G}$ as follows: 
\begin{equation}\label{eqn:internal_diffusion}
    {\tilde{\bf{x}}}_{t+1}^d(\tau)={\bf{H}}^{\mathcal{G}}(\tau){\bf{x}}_t^d.
\end{equation}
We call the vector ${\tilde{\bf{x}}}^d_{t+1}(\tau)$ the internally diffused signals from ${\bf{x}}_t^d$ by the diffusion period $\tau$ on the graph $\mathcal{G}$ over one incremental time step.

Here, we define a convex combination of the heat diffusion kernels of $K$ different predetermined diffusion periods with a set $\mathcal{T}=\{\tau^{(0)},\tau^{(1)},\cdots,\tau^{(K-1)}\}$\footnote{We predetermine the set $\mathcal{T}$ with two diffusion periods $\tau_0$ and $\tau_\infty$ that correspond to each extreme case in Eq.~(\ref{eqn:extreme_prior}), respectively. In practice, we set $\tau_0$ as the biggest one that satisfies $\left\Vert{\bf{H}}^{\mathcal{G}}(\tau)-{\bf{I}}\right\Vert_2<\epsilon$ and $\tau_\infty$ as the smallest one that satisfies $\left\Vert{\bf{H}}^{\mathcal{G}}(\tau)-1/N{\bf{1}}{\bf{1}}^T\right\Vert_2<\epsilon$ with a predefined set $\tau\in\text{linspace(-10,10,0.1)}$, where the set contains evenly spaced ($0.1$) numbers from $-10$ to $10$.
After that, we define
$\mathcal{T}=\text{logspace}(\tau_0,\tau_\infty,K),$ where the function returns $K$ evenly spaced numbers on a log scale from $\tau_0$ to $\tau_\infty$.} as
\begin{equation}\label{eqn:diffusion_matrix}
    {\bf{H}}^{\mathcal{G}}(\mathcal{T})=\sum_{\tau\in\mathcal{T}}\pi^{(\tau)}{\bf{H}}^{\mathcal{G}}(\tau),
\end{equation}
where $\sum_{\tau\in\mathcal{T}}\pi^{(\tau)}=1$.
The mixture retains the property that the total input volume is preserved through the diffusion process as shown in Appendix~\ref{apdx:volume_conservation_heat_diffusion}, i.e.,
${\bf{1}}^T{\bf{H}}^{\mathcal{G}}(\mathcal{T}){\bf{x}}_t^d={\bf{1}}^T{\bf{x}}_t^d$.

We embed heat diffusion kernels into DLM to exploit topological information of the transportation network.
The key idea is to express the transition matrix as a small variant from a mixture of diffusion kernels.
We decompose the transition matrix into the time-variant internal diffusion and 
residual
\stopchange
as follows:
\begin{equation}\label{eqn:transition_matrix_decomposition}
{\bf{H}}_t={\bf{H}}_t^{\mathcal{G}}(\mathcal{T})+\text{residual}
\end{equation}
so that the internal diffusion matrix ${\bf{H}}_t^{\mathcal{G}}(\mathcal{T})$ preserves the total traffic volume over time, i.e., ${\bf{1}}^T{\bf{H}}_t^{\mathcal{G}}(\mathcal{T}){\bf{x}}_{t}^d={\bf{1}}^T{\bf{x}}_{t}^d$.
Here, the time dependent internal transition matrix can be safely defined as in Eq.~(\ref{eqn:diffusion_matrix}) by substituting the time-invariant parameter $\pi^{(\tau)}$ for the time-variant one $\pi_t^{(\tau)}$ because of the volume conservation property.
The internal diffusion matrix represents how the current signal ${\bf{x}}_t^d$ diffuses through the transportation network (endogenous) whereas
the residual represents how much the traffic situation is getting better or worse in the next time step based on the current signal (exogenous).

With this interpretation, we model the prior distribution of the transition matrix as:
\begin{equation}\label{eqn:prior_assumption}
    f({\bf{H}}_t|\gamma_t,\Pi_t,\mathcal{G})=\mathcal{N}\left({\bf{H}}_t^{\mathcal{G}}(\mathcal{T}),\gamma_t^{-1}\right),
\end{equation}
where the precision parameter $\gamma_t$ represents how precisely the diffusion matrix explains the transition matrix and $\Pi_t=\{\pi_t^{(\tau)}|\tau\in\mathcal{T}\}$.

The decomposition allows us to utilize data more efficiently during the estimation process later.
In Eq.~(\ref{eqn:dlm}), the transition matrix is a variable to be estimated from the data. Since the dimension of this matrix is $N^2$, an increase in the number of sensors causes the estimation of more elements, which results in an overfitting problem. This is the biggest impediment to extending DLM to large networks. Still, if the structural information is set as {\it{a priori}} through Eq.~(\ref{eqn:prior_assumption}), the problem can be effectively avoided even if the number of sensors increases.
Assuming the graph $\mathcal{G}$ and the period set $\mathcal{T}$ are predefined, the internal diffusion matrix only depends on the parameters $\pi_t^{(\tau)}$. By setting the number of diffusion periods to be much smaller than that of sensors i.e., $|\mathcal{T}|\ll N$, we can describe the major part of the transition matrix by the internal diffusion matrix with a few parameters when the sampling interval (the time difference of two consecutive time indices) is relatively short, with likely preservation of the traffic volumes, i.e., ${\bf{1}}^T{\bf{x}}_{t+1}^d\approx {\bf{1}}^T{\bf{x}}_{t}^d$.
Consequently, we only need to exploit data to infer the parameters $\pi_t^{(\tau)}$ and the residual part whose norm is small with the decomposition.

\section{Prediction and inference}
This section describes how to estimate modeling parameters and predict graph signals by using the model. Both the estimation and the prediction were performed by maximizing the posterior distribution of each variable.
Especially for hyperparameters, we utilize Bayesian inference to estimate them instead of exhaustive search.

\subsection{Inference of the transition matrix}
We infer the transition matrix by maximizing its posterior distribution:
\begin{equation}
    {\hat{\bf{H}}}_t=\underset{{\bf{H}}_t}{\text{argmax }}f({\bf{H}}_t|{\bf{X}}_{t},{\bf{X}}_{t+1},\alpha_t,\gamma_t,\Pi_t,\mathcal{G}),
\end{equation}
which is proportional to the product of the prior and the likelihood by Bayes' rule:
\begin{equation}\label{eqn:bayes_rule}
\text{Posterior dist.}\propto f({\bf{H}}_{t}|\gamma_t,\Pi_t,\mathcal{G})f({\bf{X}}_{t+1}|{\bf{X}}_{t},{\bf{H}}_{t},\alpha_t).
\end{equation}

\begin{algorithm}[t!]
  \caption{Inference of parameters}\label{algo:inference_of_parameters}
  \begin{algorithmic}[1]
    \Function{Inference}{${\bf{W}},K,{\bf{X}}_{1:T}$}
    \State Set $\mathcal{T}=\text{logspace}(\tau_0,\tau_\infty,K)$
    \State Define ${\bf{L}}(\mathcal{G})$ by Eq.~(\ref{eqn:laplacian})
    \State Define the function ${\bf{H}}^\mathcal{G}(\tau)=e^{-\tau{\bf{L}}(\mathcal{G})}$
    \For{$t \in [0,T-2]$}
    \State Infer $\hat\alpha_t$, $\hat\gamma_t$ and $\hat\Pi_t$ by solving~(\ref{eqn:evidence_maximization})
    \State Infer $\hat{\bf{H}}_t$ by Eq.~(\ref{eqn:probable_H})
    \EndFor\\
    \Return $\hat{\bf{H}}_t, \forall t$
    \EndFunction
  \end{algorithmic}
\end{algorithm}

Maximizing the posterior distribution can be interpreted as balancing between the prior and likelihood of the transition matrix.
For example, if there is no topological information about sensors, the transition matrix should be inferred by considering the training dataset only. 
In this case, we can set the prior distribution as a uniform distribution, meaning that there is no strong preference for a particular value of the transition matrix; the most probable transition matrix becomes the maximum likelihood solution, which is Eq.~(\ref{eqn:likelihood_only}):
\begin{equation}
\begin{aligned}
    \hat{\bf{H}}_t|{\text{No topological info.}}&:=\bar{\bf{H}}_t\\
    &=\underset{{\bf{H}}_t}{\text{argmax }}f({\bf{X}}_{t+1}|{\bf{X}}_{t},{\bf{H}}_{t},\alpha_t)\\
    &={\bf{X}}_{t+1}{\bf{X}}_t^T({\bf{X}}_{t}{\bf{X}}_t^T)^{-1}.
\end{aligned}
\end{equation}
On the other hand, if we do not have any measurements, the most probable transition matrix should be the maximizer of the prior distribution:
\begin{equation}\label{eqn:prior_only}
    \hat{\bf{H}}_t|{\text{No measurements}}=\underset{{\bf{H}}_t}{\text{argmax }}f({\bf{H}}_t|\gamma_t,\Pi_t,\mathcal{G})={\bf{H}}_t^{\mathcal{G}}(\mathcal{T}).
\end{equation}

Since we use both prior and data measurements, the actual optimal transition matrix becomes a combination of these two.
According to the dynamic linear model, the likelihood
\begin{equation}\label{eqn:likelihood_dist}
\begin{aligned}
&f({\bf{X}}_{t+1}|{\bf{H}}_t,{\bf{X}}_t,\alpha_t)
\\&\qquad\qquad\propto e^{-\frac{1}{2}\text{tr}\{\alpha_t({\bf{X}}_{t+1}-{\bf{H}}_t{\bf{X}}_t)({\bf{X}}_{t+1}-{\bf{H}}_t{\bf{X}}_t)^T\}}
\end{aligned}
\end{equation}
and the prior
\begin{equation}\label{eqn:prior_dist}
f({\bf{H}}_t|\gamma_t,\Pi_t,\mathcal{G})\propto e^{-\frac{1}{2}\text{tr}\{\gamma_t({\bf{H}}_t-{\bf{H}}_t^{\mathcal{G}}(\mathcal{T}))({\bf{H}}_t-{\bf{H}}_t^{\mathcal{G}}(\mathcal{T}))^T\}}.
\end{equation}
Therefore, by Eq.~(\ref{eqn:bayes_rule}), (\ref{eqn:likelihood_dist}) and (\ref{eqn:prior_dist}),
\begin{equation}
\begin{aligned} f&({\bf{H}}_t|{\bf{X}}_{t+1},{\bf{X}}_t,\alpha_t,\gamma_t,\Pi_t,\mathcal{G})\\
&\propto e^{-\frac{1}{2}\alpha_t\text{tr}\{({\bf{X}}_{t+1}-{\bf{H}}_t{\bf{X}}_t)({\bf{X}}_{t+1}-{\bf{H}}_t{\bf{X}}_t)^T\}}\\
&\qquad\cdot e^{-\frac{1}{2}\gamma_t\text{tr}\{({\bf{H}}_t-{\bf{H}}_t^{\mathcal{G}}(\mathcal{T}))({\bf{H}}_t-{\bf{H}}_t^{\mathcal{G}}(\mathcal{T}))^T\}}\\
&\propto e^{-\frac{1}{2}\text{tr}\{({\bf{H}}_t-\hat {\bf{H}}_t)(\alpha_t{\bf{X}}_t{\bf{X}}_t^T+\gamma_tI)({\bf{H}}_t-\hat {\bf{H}}_t)^T\}},
\end{aligned}
\end{equation}
where 
\begin{equation}\label{eqn:probable_H}
    \begin{aligned}
    \hat{\bf{H}}_t&=({\bar{\bf{H}}}_t\alpha_t{\bf{U}}_t{\bf{\Lambda}}_t+{{\bf{H}}_t^{\mathcal{G}}(\mathcal{T})}\gamma_t{\bf{U}}_t)(\alpha_t{\bf{\Lambda}}_t+\gamma_t{\bf{I}})^{-1}{\bf{U}}_t^T\\
&=\bar{\bf{H}}_t\alpha_t{\bf{U}}_t{\bf{\Lambda}}_t(\alpha_t{\bf{\Lambda}}_t+\gamma_t{\bf{I}})^{-1}{\bf{U}}_t^T\\
&\qquad\qquad+{\bf{H}}_t^\mathcal{G}(\mathcal{T})\gamma_t{\bf{U}}_t(\alpha_t{\bf{\Lambda}}_t+\gamma_t{\bf{I}})^{-1}{\bf{U}}_t^T,
    \end{aligned}
\end{equation}
with the eigendecomposition of ${\bf{X}}_t{\bf{X}}_t^T={\bf{U}}_t{\bf{\Lambda}}_t{\bf{U}}_t^T$. 
Therefore, $f({\bf{H}}_t|{\bf{X}}_{t+1},{\bf{X}}_t,\alpha_t,\gamma_t,\Pi_t,\mathcal{G})$ is a multivariate Gaussian distribution with mean $\hat {\bf{H}}_t$ and the covariance of each row; $(\alpha_t{\bf{X}}_t{\bf{X}}_t^T+\gamma_tI)^{-1}$.

Here, we measure how much each part contributes to the transition matrix
\begin{equation}\label{eqn:data_contribution}
\begin{aligned}
    c^\text{data}_t&=\frac{w^\text{data}_t}{w^\text{data}_t+w^\text{prior}_t},\;c^\text{prior}_t&=\frac{w^\text{prior}_t}{w^\text{data}_t+w^\text{prior}_t}
\end{aligned}
\end{equation}
by defining the weight of each part
\begin{equation}
\begin{aligned}
    w^\text{data}_t&=\left\Vert\alpha_t{\bf{U}}_t{\bf{\Lambda}}_t(\alpha_t{\bf{\Lambda}}_t+\gamma_t{\bf{I}})^{-1}{\bf{U}}_t^T\right\Vert_F,\\
    w^\text{prior}_t&=\left\Vert\gamma_t{\bf{U}}_t(\alpha_t{\bf{\Lambda}}_t+\gamma_t{\bf{I}})^{-1}{\bf{U}}_t^T\right\Vert_F.
\end{aligned}
\end{equation}
Note that these weights depend on the precision parameters $\alpha_t$ and $\gamma_t$.
If the data precision parameter $\alpha_t$ is relatively large compared to $\gamma_t$,
then $c^\text{data}_t>c^\text{prior}$, meaning that the contribution of data measurements is larger than that of the prior information.

\subsection{Inference of other parameters}

For the next step, we infer parameters $\alpha_t$, $\gamma_t$, and $\Pi_t$.
Similar to inferring the most probable transition matrix, we infer the most probable $\alpha_t$, $\gamma_t$, and $\Pi_t$ by maximizing the following posterior distribution:
\begin{equation}
    \hat\alpha_t,\hat\gamma_t,\hat{\Pi}_t=\underset{\alpha_t,\gamma_t,{\Pi}_t}{\text{argmax }}f(\alpha_t,\gamma_t,{\Pi}_t|{\bf{X}}_{t+1},{\bf{X}}_t).
\end{equation}
Setting the prior distribution $f(\alpha_t,\gamma_t,{\Pi}_t)$ as a uniform distribution based on the assumption that there is no preference for a certain value for these parameters before inferring, the objective changes to maximize {\it{evidence}} $f({\bf{X}}_{t+1}|{\bf{X}}_t,\alpha_t,\gamma_t,{\Pi}_t)$~\cite{mackay1992bayesian} since
\begin{equation}
\begin{aligned}
f(\alpha_t,\gamma_t,{\Pi}_t|{\bf{X}}_{t+1},{\bf{X}}_t)&\propto f({\bf{X}}_{t+1}|{\bf{X}}_t,\alpha_t,\gamma_t,{\Pi}_t)f(\alpha_t,\gamma_t,{\Pi}_t)
\\&\propto f({\bf{X}}_{t+1}|{\bf{X}}_t,\alpha_t,\gamma_t,{\Pi}_t).\end{aligned}
\end{equation}
In Appendix~\ref{apdx:evidence}, we show that the evidence is
\begin{equation}\label{eqn:weighted_average}
\begin{aligned}
    f({\bf{X}}_{t+1}|&{\bf{X}}_t,\alpha_t,\gamma_t,{\Pi}_t)\\
    &=\int f({\bf{X}}_{t+1}|{\bf{X}}_t,{\bf{H}}_t,\alpha_t) f({\bf{H}}_t|\gamma_t,{\Pi}_t)d{\bf{H}}_t\\
    &=\mathcal{N}({\bf{H}}_t^{\mathcal{G}}(\mathcal{T}){\bf{X}}_t,\alpha_t^{-1}{\bf{I}}+\gamma_t^{-1}{\bf{X}}_t^T{\bf{X}}_t).
\end{aligned}
\end{equation}
Therefore, we infer the most probable hyper-parameters by maximizing the log-evidence with a quasi-newton method (L-BFGS-B~\cite{byrd1995limited}):
\begin{equation}\label{eqn:evidence_maximization}
\begin{aligned}
    &\underset{\alpha_t,\gamma_t,{\Pi}_t}{\text{maximize }}& &\log\mathcal{N}({\bf{H}}_t^{\mathcal{G}}(\mathcal{T}){\bf{X}}_t,\alpha_t^{-1}{\bf{I}}+\gamma_t^{-1}{\bf{X}}_t^T{\bf{X}}_t)
    \\&\text{subject to }& & 0\le\pi_t^{(\tau)}\le 1\;\forall \tau\in\mathcal{T},\;0<\alpha_t,\;0<\gamma_t,
    \\& & &\sum_{\tau\in\mathcal{T}}\pi_t^{(\tau)}=1.
\end{aligned}
\end{equation}
Algorithm~\ref{algo:inference_of_parameters} summarizes the inference processes.

We emphasize that parameter inference through evidence maximization prevents overfitting of the transition matrix to either data measurements or prior information.
In Eq.~(\ref{eqn:weighted_average}) we calculate the evidence by marginalizing the transition matrix. In other words, we set the transition matrix as a {\it{random variable}} instead of fixing it as a representative value, e.g., maximum likelihood estimator. Noting that these parameters determine the contributions of measurements and priors when the transition matrix is estimated in Eq.~(\ref{eqn:probable_H}), the marginalization process automatically penalizes the transition matrix to avoid the extreme cases~\cite{mackay1992bayesian}.

\begin{algorithm}[t!]
\caption{Prediction of traffic features ($h$-steps ahead)}\label{algo:prediction}
\begin{algorithmic} 
\Function{Prediction}{${\bf{x}}_t^d,h, \hat{\bf{H}}_t,\cdots,\hat{\bf{H}}_{t+h-1}$}
\State Set ${\bf{p}}={\bf{x}}_t^d$
\For{$i \in [0,h-1]$}
\State Set ${\bf{p}}=\hat{\bf{H}}_{t+i}{\bf{p}}$ 
\EndFor
\State ${\bf{x}}_{t+h|t}={\bf{p}}$\\
\Return ${\bf{x}}_{t+h|t}$
\EndFunction
\end{algorithmic}
\end{algorithm}

\subsection{Prediction of traffic features}
Prediction of traffic features is performed by extracting and exploiting as much information as possible from measurements and prior knowledge. Mathematically, we can express a traffic signal that we want to predict as a random variable since the signal defined in the future is entirely unknown. In this paper, therefore, we try to infer the probability density function of the signal ${\bf{x}}_{t+h}^d$
\begin{equation}\label{eqn:basic_fw}
    f({\bf{x}}_{t+h}^d|{\bf{x}}_{t}^d,{\bf{x}}_{t-1}^d\cdots,\mathcal{G}),
\end{equation} 
where the time indices $t$ and $t+h$ represent respectively the current time and the future time index ($h$-steps ahead) that we want to predict. In the expression, the probability density function is conditioned by the signals $\{{\bf{x}}_t^d,{\bf{x}}_{t-1}^d,\cdots\}$ and the graph $\mathcal{G}$ that represents a set of measurements and prior structural information, respectively.

\begin{figure}[!t]
\centering
 {{\includegraphics[width=0.9\columnwidth]{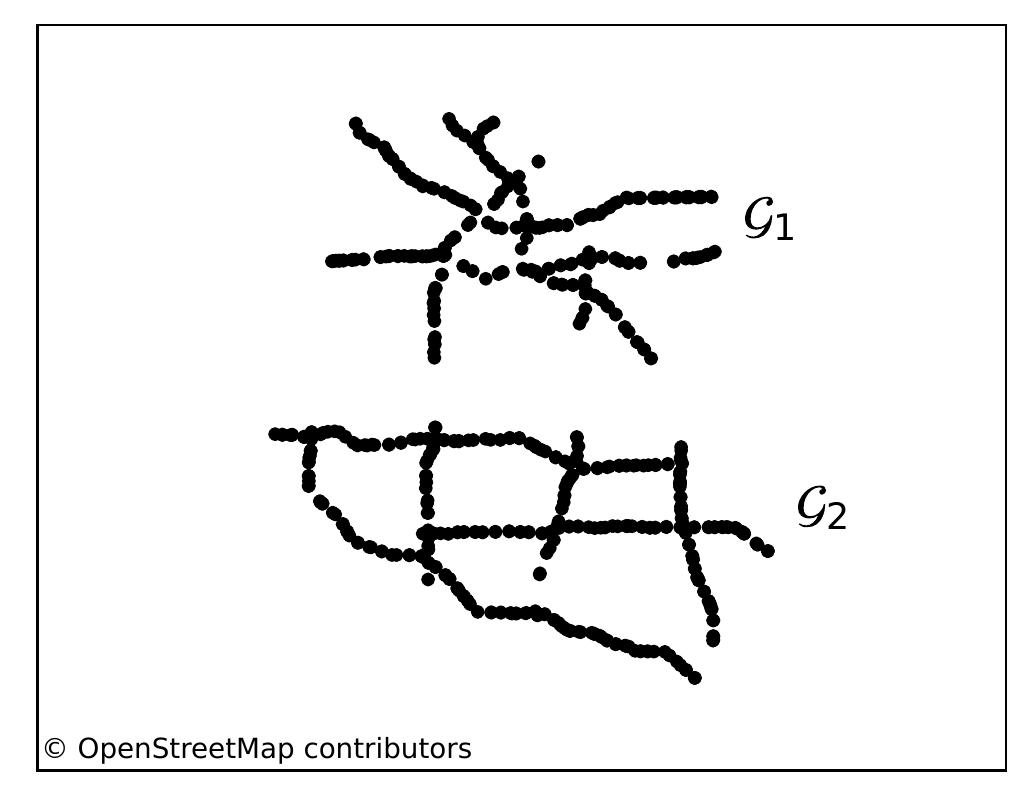}}}
\caption{Transportation sensor networks (District 7 area in California) that are used for evaluating the proposed method.}\label{fig:test_sites}
\end{figure}

In reality, it is common to limit the number of measurements to a fixed-sized one in a training set. 
In addition to the training set that contains measurements apart from the day to be predicted, it is crucial to keep measurements just before $t$, as the temporal correlation is strong when the time difference is small.
As a result, we estimate the density function that is conditioned by a training set, the $p$-most recent measurements, and the graph $\mathcal{G}$:
\begin{equation}\label{eqn:limited_pdf_by_training_set}
    f({\bf{x}}_{t+h}^d|{\bf{x}}_{t}^d,{\bf{x}}_{t-1}^d,\cdots,{\bf{x}}_{t-(p-1)}^d,{\bf{X}}_{0:T-1},\mathcal{G}),
\end{equation}
where the training set ${\bf{X}}_{0:T-1}$ contains signals from $t=0$ to $t=T-1$ of multiple days $d\in[0,m-1]$.
The dynamic linear model further simplifies the distribution~(\ref{eqn:limited_pdf_by_training_set}) as follows
\begin{equation}
    f({\bf{x}}_{t+h}^d|{\bf{x}}_{t}^d,{\bf{X}}_{t:t+h},\mathcal{G})
\end{equation}
because of the temporal locality of the model.

We define a predictor ${\bf{x}}_{t+h|t}^d$ at the time step $t$ for the horizon $h$ as the maximizer of the probability density function
 \begin{equation}
     {\bf{x}}^d_{t+h|t}:=\underset{{\bf{x}}^d_{t+h}}{
     \text{argmax}}
     f({\bf{x}}_{t+h}^d|{\bf{x}}_{t}^d,{\bf{X}}_{t:t+h},\mathcal{G}).
 \end{equation}
In other words, we define the predictor ${\bf{x}}^d_{t+h|t}$ as the most probable ${\bf{x}}^d_{t+h}$ based on the current measurement vector ${\bf{x}}^d_t$, the training set ${\bf{X}}_{t:t+h}$, and the graph $\mathcal{G}$.

\begin{prop}\label{prop:1}
$f({\bf{x}}^d_{t+h}|{\bf{x}}^d_t,{\bf{X}}_{t:t+h},\mathcal{G})$ is a Gaussian distribution that has the mean vector $\hat {\bf{H}}_{t+h-1}\cdots\hat {\bf{H}}_t{\bf{x}}^d_t$ assuming $f({\bf{H}}_t|{\bf{X}}_{t},{\bf{X}}_{t+1},\alpha_t,\gamma_t,\Pi_t,\mathcal{G})=\delta({\bf{H}}_t-\hat {\bf{H}}_t)$, where the Dirac delta function $\delta(x)=1$ when $x=0$ and $\delta(x)=0$, otherwise. The most probable transition $\hat{\bf{H}}_t$ is the maximizer of the posterior distribution $f({\bf{H}}_t|\cdot)$.
\end{prop}
\begin{proof}
See Appendix~\ref{apdx:prediction}.
\end{proof}
Since the mean value of a Gaussian distribution maximizes the distribution, the optimal predictor is
\begin{equation}\label{eqn:pred_proposed}
    \begin{aligned}
        {\bf{x}}^d_{t+h|t}&=\hat {\bf{H}}_{t+h-1}\cdots\hat {\bf{H}}_t{\bf{x}}_t
        :=\hat {\bf{H}}_{t+h-1\leftarrow t}{\bf{x}}^d_t.
    \end{aligned}
\end{equation}
Therefore, the most probable signal ${\bf{x}}_{t+h}^d$ is the successive propagation of the current measurement vector ${\bf{x}}_t^d$ through the most probable transition matrices 
that coincides with a straightforward computation with Eq.~(\ref{eqn:dlm}) ignoring the noise term. Therefore, the prediction for any horizon is just a matrix multiplication.
Algorithm~\ref{algo:prediction} summarizes this.

\begin{figure}[!t]
   \centering
   {\includegraphics[width=0.9\columnwidth]{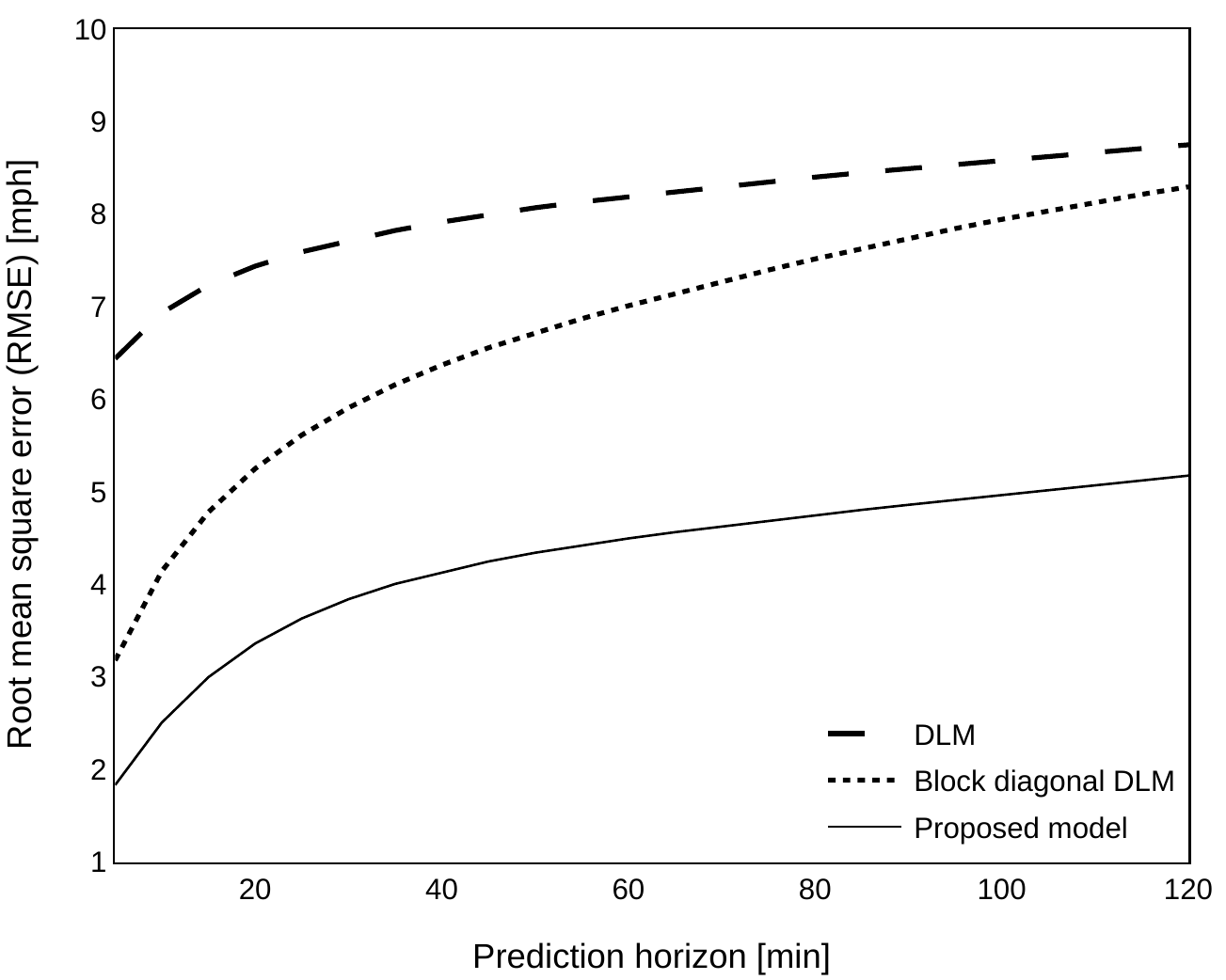}}
   \caption{Prediction accuracy (RMSE) for the three different models on the PEMS-BAY dataset. Each model represents respectively a single DLM (without topological information), separate multiple DLMs for each freeway, and the proposed model (a DLM with topological information).}\label{fig:single_vs_network}
\end{figure}
\begin{figure}[!t]
   \centering
   {\includegraphics[width=0.9\columnwidth]{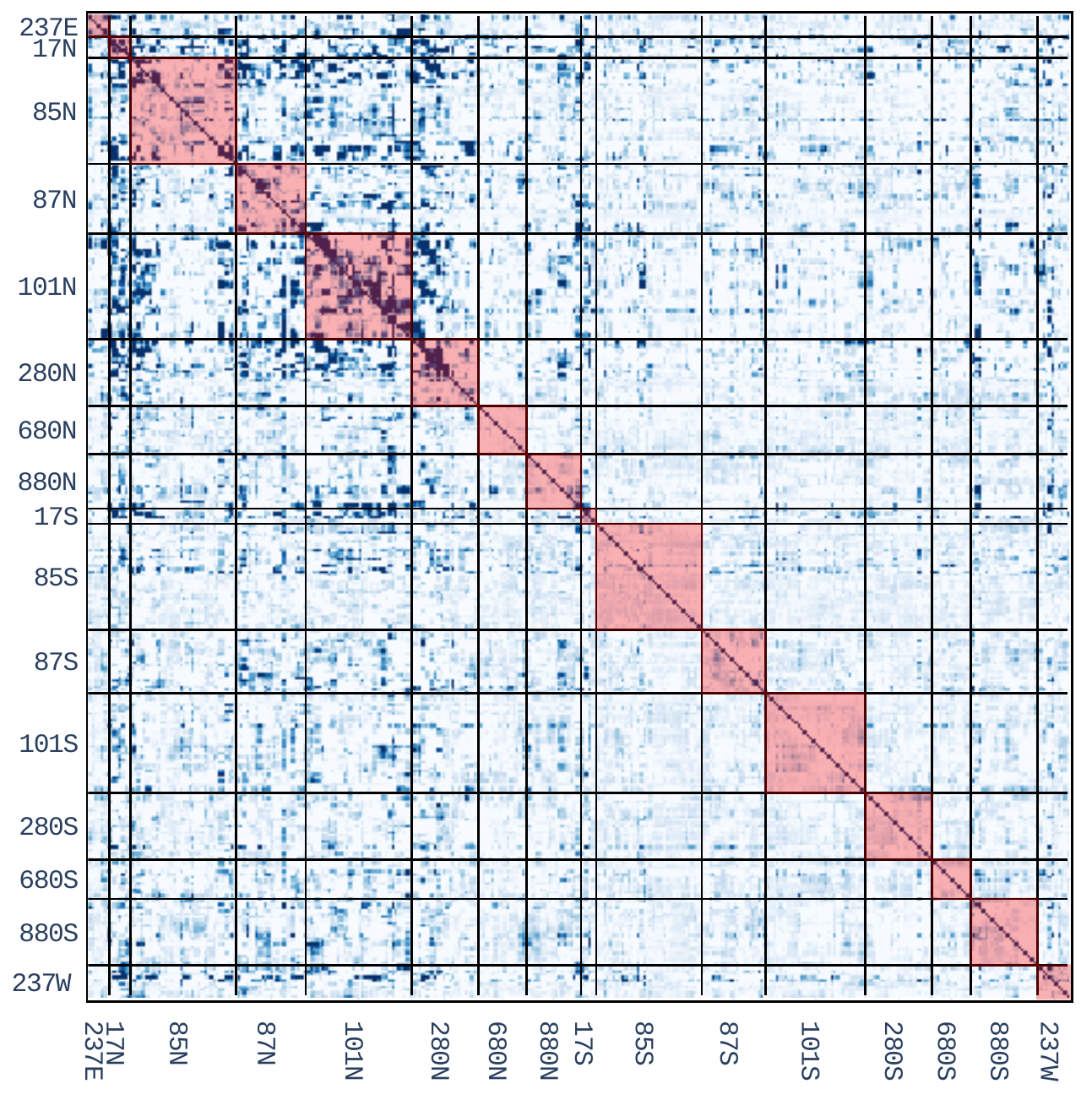}}
   \caption{The heatmap of the elements in an estimated transition matrix ${\bf{H}}_t$ of the proposed model.
   Darker colors represent larger absolute values.
   The sensors are grouped by freeways and ordered from upstream to downstream within each freeway. Each axis shows the name of the freeways. The sensors' correlations within the same freeway are represented as red-shaded areas (block-diagonal elements of the matrix). The separate multiple DLMs only use block diagonal elements in the matrix.
   }
   \label{fig:h_morning}
\end{figure}

\section{Experiments}
\subsection{Settings}

The proposed method was evaluated on different transportation networks. Figure~\ref{fig:test_sites} shows the networks ($\mathcal{G}_1$ and $\mathcal{G}_2$) consisting of respectively 288 and 357 sensors with multiple freeways that are connected through ramps. They experience significant levels of congestion in the morning and evening peaks at various locations. 
These networks connect many origins and destinations with complex demand profiles, creating propagation of congestion that is different in duration, size, and time of occurrence. 
The PEMS-BAY dataset was also used as a benchmark to compare with other state-of-the-art models~\cite{li2018diffusion,wu2019graph}. This data set consists of data measured from 325 sensors (Fig.~\ref{fig:sensornetwork}(a)) on the freeways of San Francisco Bay area. The training and test dataset were constructed in the same way as \cite{li2018diffusion,wu2019graph} to achieve a fair comparison.

The sampling interval of each dataset is 5 minutes by default, and in the following subsection, it is downsampled to 10 and 15 minutes, respectively, for a specific experiment.
Both datasets of networks $\mathcal{G}_1$ and $\mathcal{G}_2$ contain 209 days of speed data, and each of those is divided into a training set and a test set at an 8:2 ratio by default.
Another ratio is applied in Section~\ref{sec:network_prior} for a specific experiment.

We used the root mean square error (RMSE) as an error metric to measure the accuracy of prediction since the solution in Eq.~(\ref{eqn:pred_proposed}) is also the optimal under the minimum mean squares error (MMSE) sense~\cite{kwak2020travel}.
The RMSE of a method with the prediction horizon $h$ is defined as
\begin{equation}
    \text{RMSE}(h,{\text{method}})=\sqrt{\text{mean}({\bf{x}}_{t+h|t}^{\text{method}}-{\bf{x}}_{t+h})^2},
\end{equation}
where the mean value is evaluated over all $t$ in the test set.

For prediction horizons, we set from 5 minutes to 120 minutes every 5 minutes. In our previous work~\cite{kwak2020travel}, on a freeway with a total length of about 60 miles (similar to the longest path of the networks considered here), the actual travel time is about 70 minutes under usual congestion. In the most severe congestion, the maximum travel time is about 100 minutes, and accordingly, we set the maximum prediction horizon to 120 minutes.

All datasets were normalized using the mean and standard deviation of each sensor in the training set.
For a reference, we defined a baseline method that predicts future traffic features assuming that the current traffic does not change over time, i.e., ${\bf{x}}_{x+h|t}^{\text{baseline}}={\bf{x}}_t$.

\subsection{Analysis of network prior}\label{sec:network_prior}
\begin{figure*}[!ht]
   \centering
   \subfigure[Prediction accuracy (97/73 days for training/test set; 43:37 ratio)]{{\includegraphics[width=0.3\textwidth]{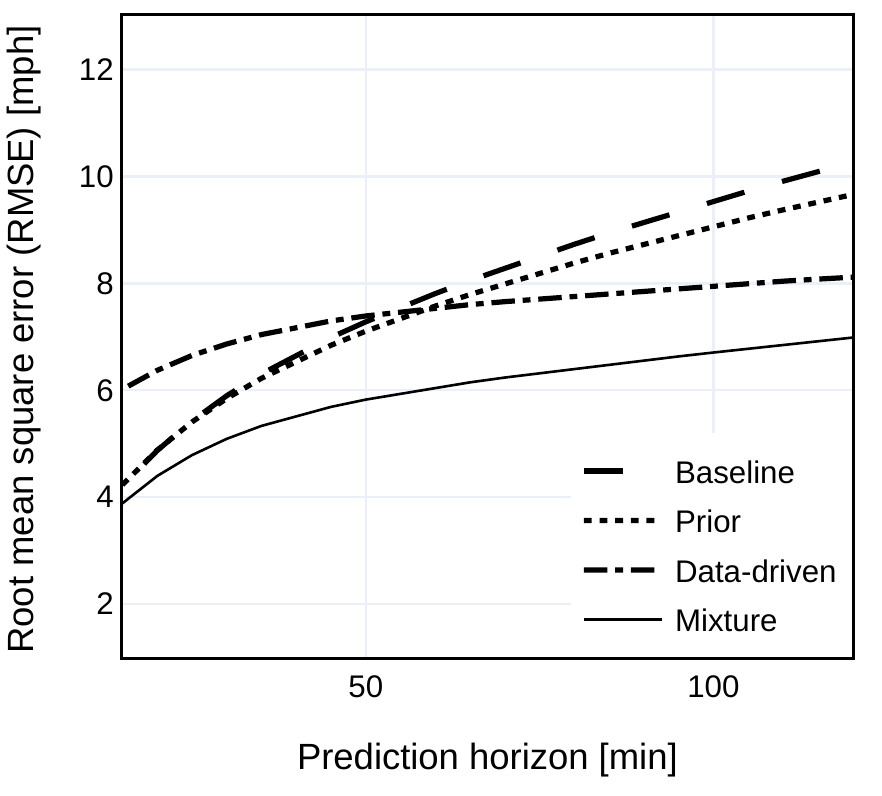}}}\hfill
   \subfigure[Prediction accuracy (194/73 days for training/test set; 73:27 ratio)]{{\includegraphics[width=0.3\textwidth]{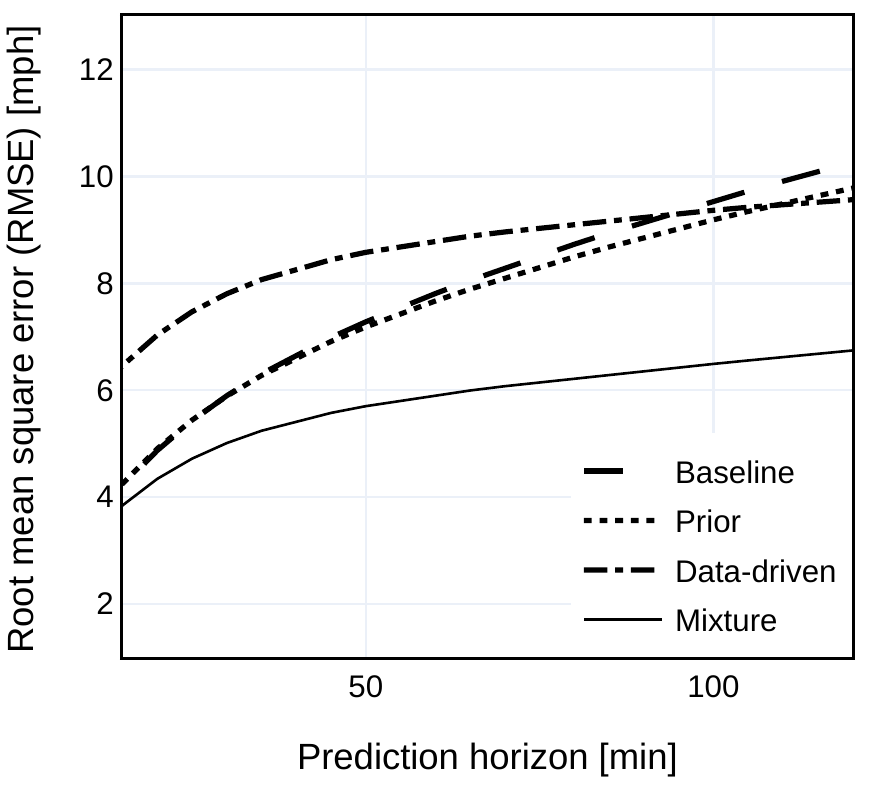}}}\hfill
   \subfigure[Prediction accuracy (292/73 days for training/test set; 8:2 ratio)]{{\includegraphics[width=0.3\textwidth]{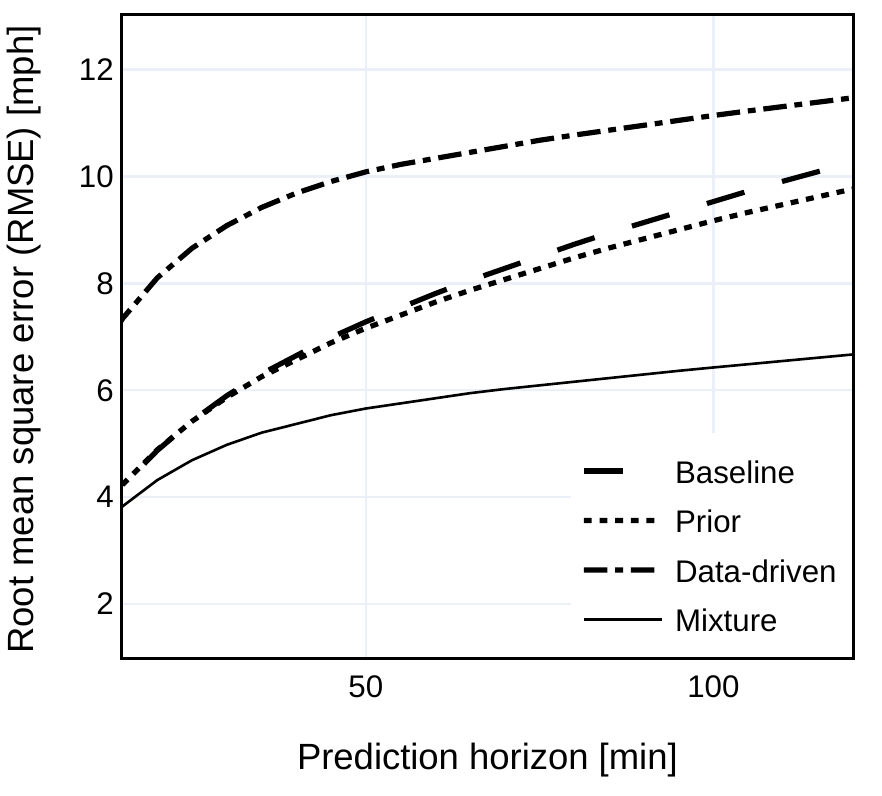}}}
   \subfigure[$c^{\text{data}}_t$ (97/73 days for training/test set; 43:37 ratio)]{{\includegraphics[width=0.3\textwidth]{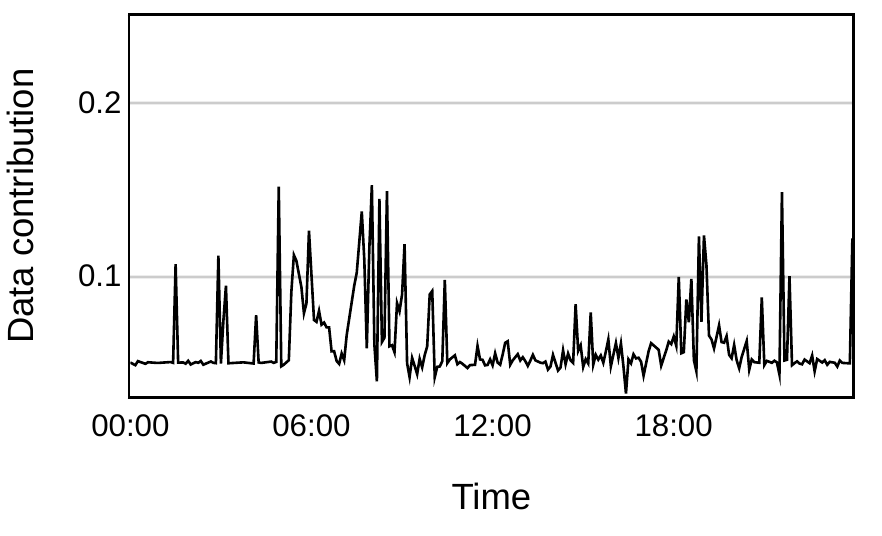}}}\hfill
   \subfigure[$c^{\text{data}}_t$ (194/73 days for training/test set; 73:27 ratio)]{{\includegraphics[width=0.3\textwidth]{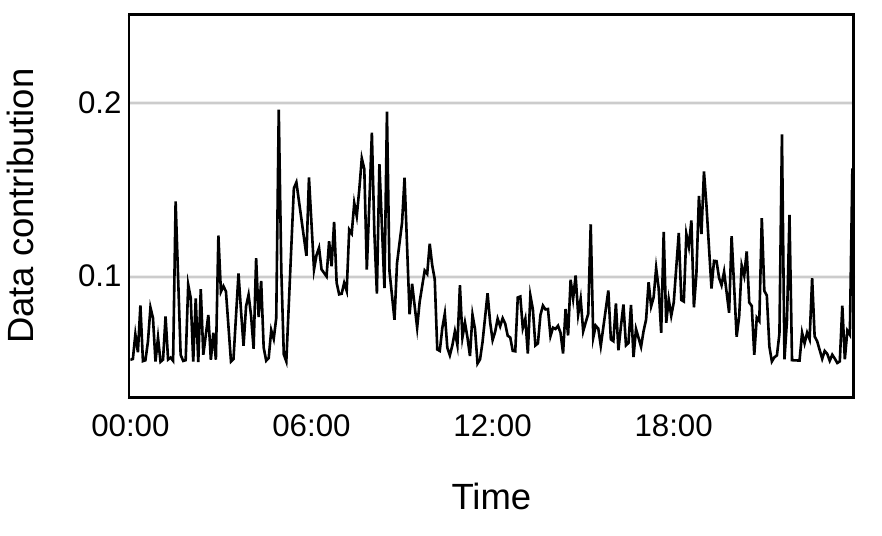}}}\hfill
   \subfigure[$c^{\text{data}}_t$ (292/73 days for training/test set; 8:2 ratio)]{{\includegraphics[width=0.3\textwidth]{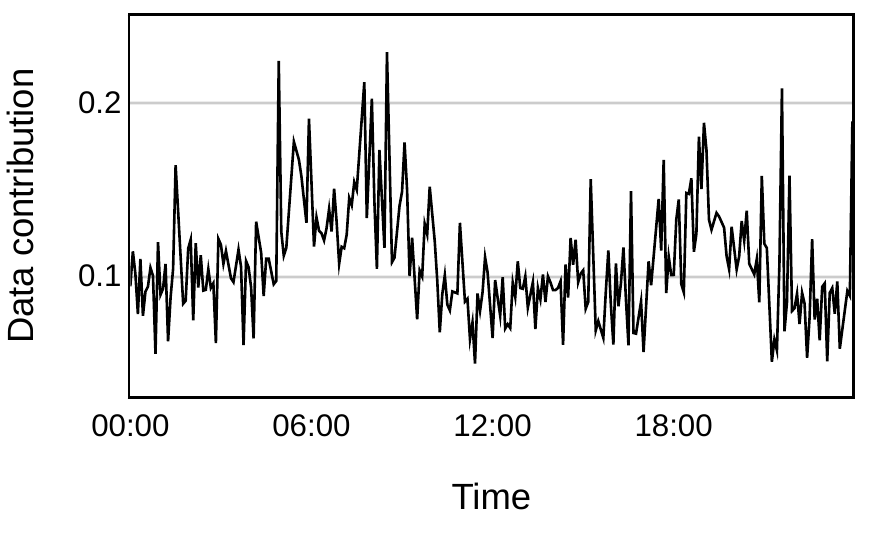}}}
   \caption{Accuracy of the prediction and the data contribution for different training-test set ratio. 
   The baseline method predicts future traffic features assuming that the current traffic does not change over time, i.e., ${\bf{x}}_{x+h|t}^{\text{baseline}}={\bf{x}}_t$.
   }\label{fig:rmse_performance_prior_data}
\end{figure*}

In this section, we show how network prior information contributes to predictive performance. 
Our model generalizes the DLM~\cite{kwak2020travel} to extend the model for a more extensive sensor network using the sensor's topology structure. 
When the DLM is simply used in an extensive network without topology structure information, an overfitting problem can occur. 
We introduce the three following setups to evaluate how well the proposed model utilizes the topology structure avoiding the overfitting problem,
\begin{enumerate}
\item a single DLM for the entire sensor network (without topological information),
\item separate DLMs ($K=5$) for each freeway (block-diagonal DLM),
\item and the proposed model that is a single DLM ($K=5$) with topological information.
\end{enumerate}

As shown in Fig.~\ref{fig:single_vs_network}, the proposed model shows the best performance, followed by block-diagonal DLM and single DLM without topological information.
The proposed model induces the sensor's topological information through heat diffusion kernels to give weights to each element of this transition matrix and focus on estimating more essential components, resulting in it as a sparse matrix, as shown in Fig.~\ref{fig:h_morning}. As a result, it shows excellent performance in long-term prediction by effectively estimating off-diagonal elements (correlation between signals of sensors installed on different freeways) while avoiding the overfitting problem.
In the case of the model with a single DLM, all elements of this matrix are estimated using historical data, while in the case of the model with separate DLMs, only the block diagonal elements are estimated (red shaded area). Therefore, since the former one needs to estimate a much larger number of elements from the data than the latter, an overfitting problem may occur. 
In contrast, in the separate DLMs, the historical data cannot be fully utilized due to the lack of association between sensors belonging to different freeways. In particular, this insufficiency causes degradation of long-term predictions as congestion propagates slowly from one freeway to others.

The low prediction error is obtained only when the topological information is optimally implanted into the DLM. Bayesian inference in our model is the key component to support this process, as it optimally estimates various parameters that characterize the mixing ratio between data and prior, which respectively correspond to DLM and topological information. We set up the following experiment to find test the effectiveness of this estimation method:
\begin{enumerate}
\item the model with measurements (Eq.~(\ref{eqn:likelihood_only})),
\item the model with topological information (Eq.~(\ref{eqn:prior_only})),
\item and the model with both topological information and measurements (Eq.~(\ref{eqn:probable_H})).
\end{enumerate}
For all the above models, we set three different cases that are characterized by different sizes of the training sets with the same test set.

Figures~\ref{fig:rmse_performance_prior_data}(a)-(c) show the prediction accuracy of each case. Interestingly, the model with measurements produced smaller errors when the size of the training set is smaller.
The reason is that each training set period is close to that of the test set with respect to time, which means larger training sets contain measurement that are far from those in the test sets. This may distort the inference process as traffic measurements have seasonal patterns.
On the other hand, the model using only the topological information showed poor performance in predicting the far future because mixture kernels do not represent well the change in traffic conditions due to the volume preservation characteristic.
The model with both topological information and measurements showed the best performance and similar outputs regardless of the size of the training set. 
It shows that Bayesian inference estimates parameters $\alpha_t$ and $\gamma_t$ in Eq.~(\ref{eqn:probable_H}) optimally, extracting maximal information both from data and prior.

Figures~\ref{fig:rmse_performance_prior_data}(d)-(f) show the data contribution which is defined in Eq.~(\ref{eqn:data_contribution}) of the mixture model. As the size of the training set increases, the data contribution increases since the larger training set can generalize measurements more easily. Another important aspect from the results is that the data contribution increases during peak periods such as morning and evening peaks since the traffic volume is most likely not preserved during these periods (therefore, it is difficult to explain it only with diffusion processes).

\subsection{Analysis of different diffusion periods}\label{sec:different_diffusion_periods}
\begin{figure*}[!ht]    
   \centering
   \subfigure[Prediction accuracy ($T_s=$5 min.)]{{\includegraphics[width=0.3\textwidth]{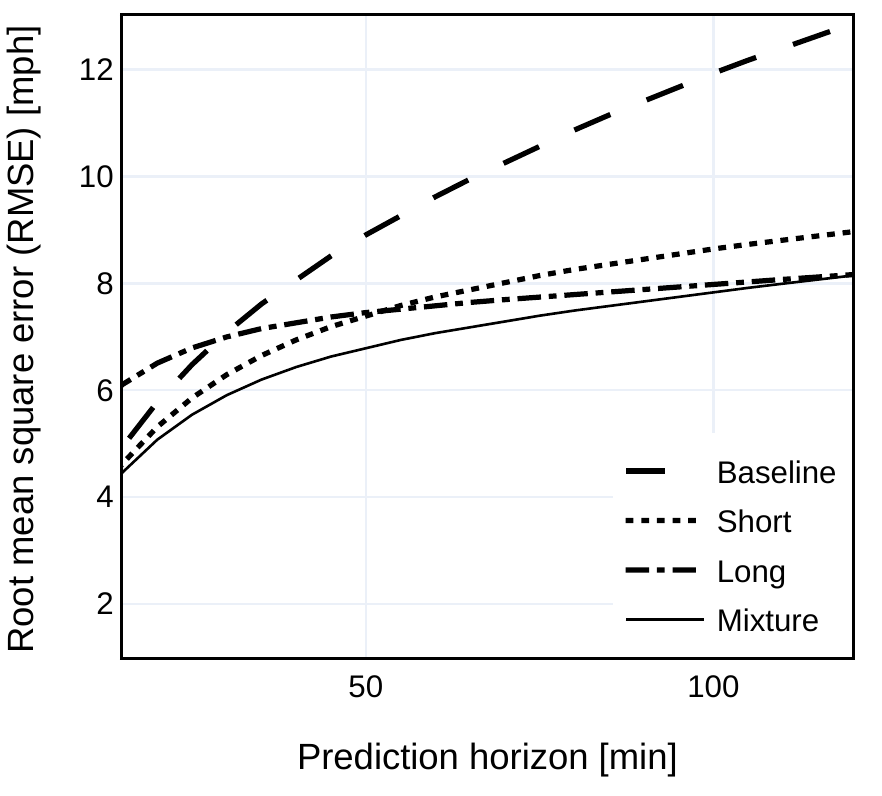}}}\hfill
   \subfigure[Prediction accuracy ($T_s=$10 min.)]{{\includegraphics[width=0.3\textwidth]{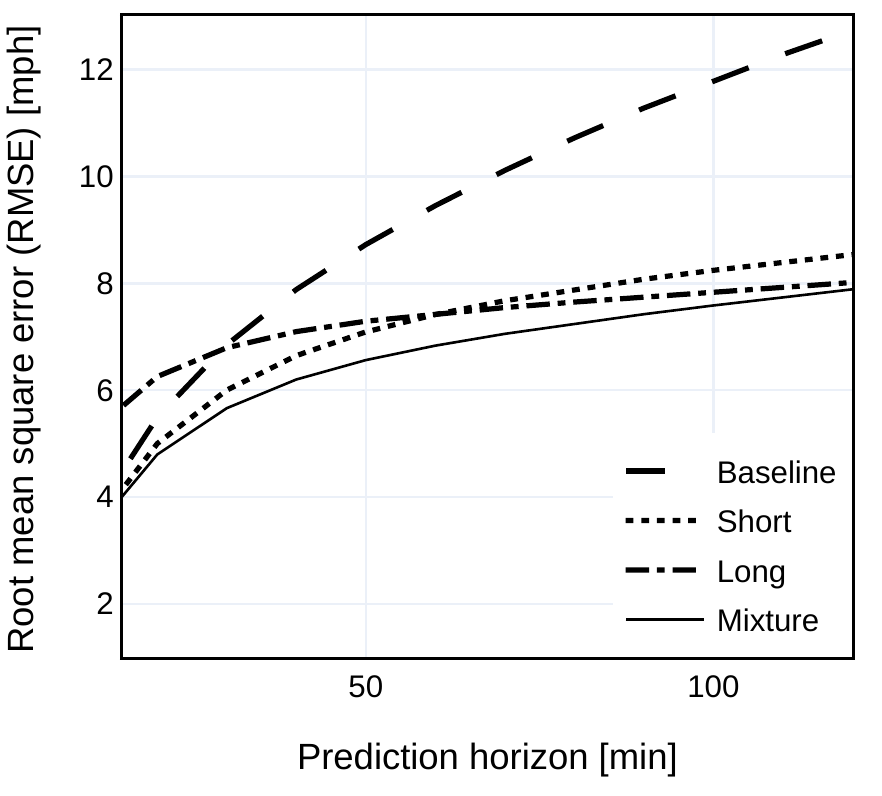}}}\hfill
   \subfigure[Prediction accuracy ($T_s=$15 min.)]{{\includegraphics[width=0.3\textwidth]{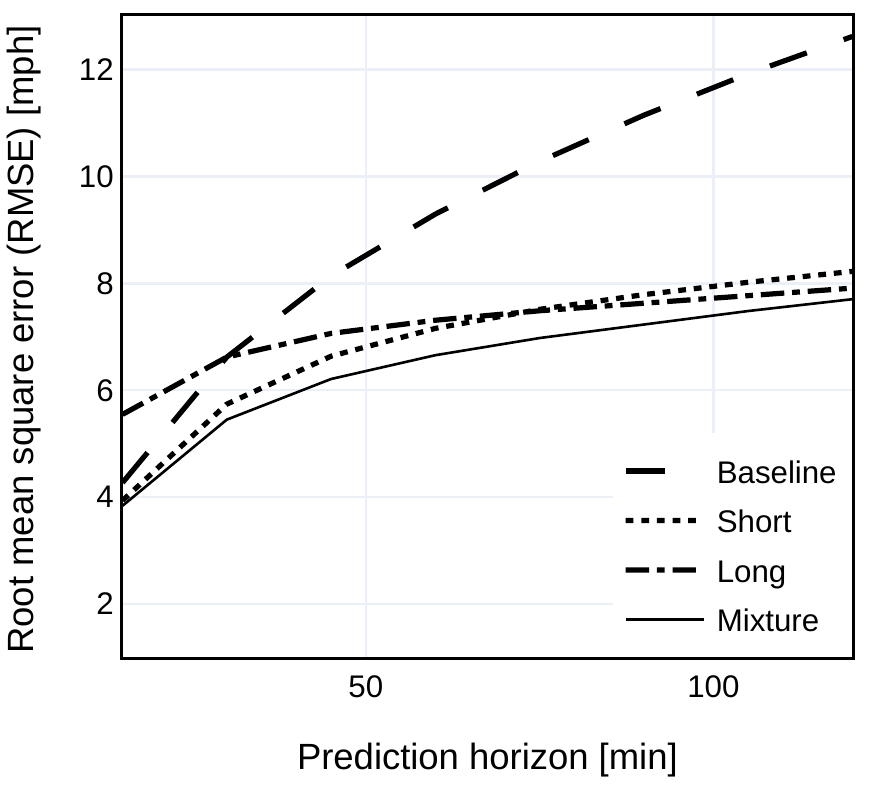}}}
   \subfigure[$\pi_t^{(\tau_0)}/\pi_t^{(\tau_\infty)}$ ($T_s=$5 min)]{{\includegraphics[width=0.3\textwidth]{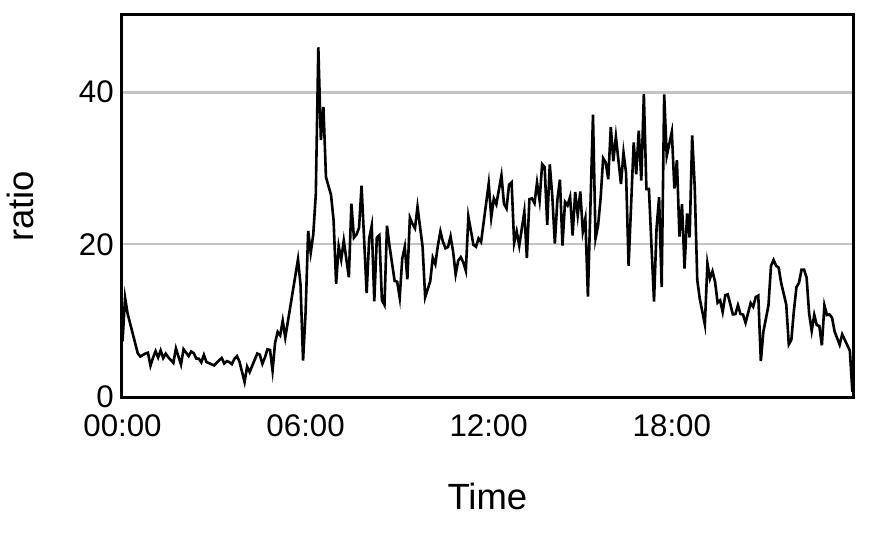}}}\hfill
   \subfigure[$\pi_t^{(\tau_0)}/\pi_t^{(\tau_\infty)}$ ($T_s=$10 min.)]{{\includegraphics[width=0.3\textwidth]{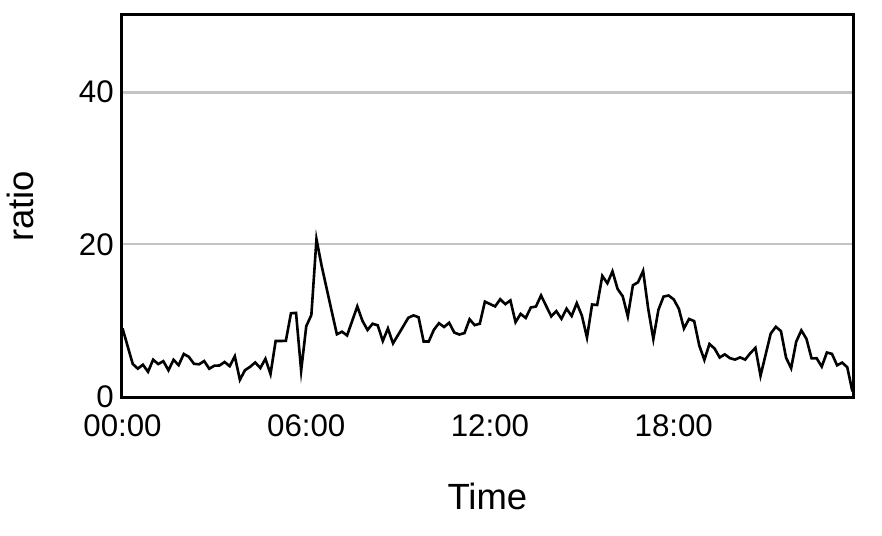}}}\hfill
   \subfigure[$\pi_t^{(\tau_0)}/\pi_t^{(\tau_\infty)}$ ($T_s=$15 min.)]{{\includegraphics[width=0.3\textwidth]{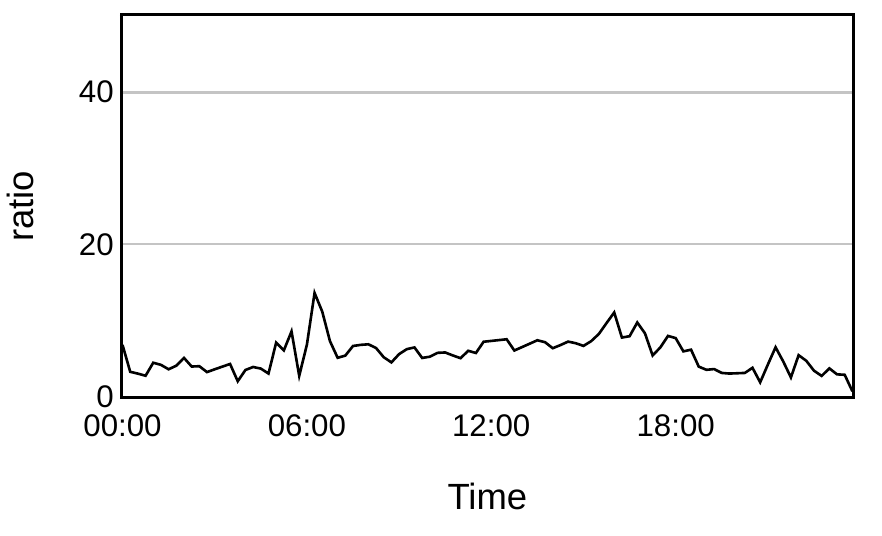}}}
   \caption{Accuracy of the prediction and ratio of the short and long diffusion processes for the same test set with different time intervals.
   The baseline method predicts future traffic features assuming that the current traffic does not change over time, i.e., ${\bf{x}}_{x+h|t}^{\text{baseline}}={\bf{x}}_t$.
   }\label{fig:rmse_performance_long_short}
\end{figure*}
We evaluated the proposed method with different diffusion processes (short, long, and mixture of both) in order to examine how the model of Eq.~(\ref{eqn:diffusion_matrix}) performs in different settings. The transition matrix $\hat{\bf{H}}_t$ was set from Eq.~(\ref{eqn:probable_H}) with three different diffusion priors:
\begin{enumerate}
  \item ${\bf{H}}_t^\mathcal{G}(\mathcal{T})={\bf{H}}^\mathcal{G}(\tau_0)=\lim_{\tau\rightarrow 0}e^{-\tau {\bf{L}}(\mathcal{G})}$ (short diffusion kernel; identity mapping),
  \item ${\bf{H}}_t^\mathcal{G}(\mathcal{T})={\bf{H}}^\mathcal{G}(\tau_\infty)=\lim_{\tau\rightarrow \infty}e^{-\tau {\bf{L}}(\mathcal{G})}$ (long diffusion kernel; averaging),
  \item and ${\bf{H}}_t^\mathcal{G}(\mathcal{T})=\pi_t^{(\tau_0)}{\bf{H}}^\mathcal{G}(\tau_0)+\pi_t^{(\tau_\infty)}{\bf{H}}^\mathcal{G}(\tau_\infty)$ (mixture of short and long diffusion kernels).
\end{enumerate}

We also set three different cases that are characterized by different sampling intervals ($T_s$), 5, 10, and 15 minutes. 
The sampling interval indicates the time duration that corresponds to the one-time incremental (the difference between $t+1$ and $t$). 
The sampling interval is related to the diffusion period $\tau$ as a diffusion kernel expresses how traffic signals diffuse through a graph within a sampling interval.

Figures~\ref{fig:rmse_performance_long_short}(a)-(c) show the prediction accuracy of each diffusion prior on the transportation network~$\mathcal{G}_1$ with the three different sampling intervals.
The predictor with the long diffusion process showed relatively poor performance compared to the baseline method for small prediction horizons, but it was improved when prediction horizons become larger. On the other hand, the one with the short diffusion process showed relatively good performance compared to the baseline method for all prediction horizons; however, it had insufficient performance for large prediction horizons compared to the one with the long diffusion process. The mixture model takes advantage of the two extreme cases, significantly improving the performance for both small and large prediction horizons. Specifically, around 50~minutes prediction horizon in Fig.~\ref{fig:rmse_performance_long_short}(a), the performance of the mixture model is noticeably better than the others, meaning that a mixture of poor predictors can produce a good performance.

We emphasize that the distribution of the diffusion processes ($\Pi_t$) was determined optimally by Bayesian inference. Figures~\ref{fig:rmse_performance_long_short}(d)-(f) show the ratio of the coefficients $\pi_t^{(\tau_0)}$ and $\pi_t^{(\tau_\infty)}$ in the mixture model that corresponds to the short and long diffusion processes, respectively. Although the short diffusion process dominates the whole process, as shown in the figures, the small portion of the long diffusion process contributes to the improvement. 
More importantly, the ratio becomes smaller when the sampling interval increases. It shows that Bayesian inference performs well in optimally determining parameters, since the performance of the mixture model stays similar when the sampling interval is changed.

We also emphasize that the ratio depends on time. 
For example, during the early morning, the diffusion kernel with long diffusion period ($\tau_\infty$) contributes more to the prediction performance
although short diffusion (identity mapping) seems to be a more reasonable choice as there are few changes in traffic during that time. 
However, if the signal values are relatively uniform (in the case of a traffic signal at early morning), taking an average can remove noise while minimizing signal distortion as
${\bf{x}}_{t+1}\approx {\bf{x}}_t\text{ (identity)} \approx \frac{1}{N}{\bf{1}}{\bf{1}}^T{\bf{x}}_t\text{ (averaging; robust to noise)}$.

\subsection{Comparison with state-of-the-art technologies}
We compare the proposed method with other methods using a benchmark dataset: PEMS-BAY dataset~\cite{li2018diffusion}.
For a fair comparison, we use the same settings which are defined in \cite{li2018diffusion} (also same in \cite{xu2018graph})\footnote{Our code is available at: \url{https://github.com/semink/lsdlm/}}.
The models used for the comparison are as follows.

\subsubsection{FC-LSTM (Fully Connected Long Short-Term Memory)} 
This model has been used as a representative reference for time-sequence modeling in deep learning~\cite{hochreiter1997long}.
In general, the LSTM module extracts correlations of signals farther apart in time than the RNN structure. However, this model's disadvantage is that spatial correlations can only be expected to learn directly from data as there is no separate module for extracting spatial relationships of signals.
The RMSE score for PEMS-BAY dataset is retrieved from \cite{li2018diffusion}.
\subsubsection{STGCN} \textcite{yu2018spatio} extracted spatial features with Graph Convolutional Neural Network (CNN) utilizing spectral graph convolution in graph theory. After that, they attached Gated CNN block to extract temporal features.
\subsubsection{DCRNN (Diffusion Convolution Recurrent Neural Network)} 
\textcite{li2018diffusion} constructed a successful predictor by extracting the signal's spatial features from the underlying graph structure by diffusion convolutional layers. 
Compared to STGCN, they designed the filter in the spatial domain directly rather than the graph spectral domain. 
The authors combine this diffusion module to Gated Recurrent Unit (GRU) which is a Recurrent Neural Network (RNN) variant.
\subsubsection{Graph WaveNet}
\textcite{xu2018graph} improved DCRNN by using dilated 1D convolution (also called WaveNet) to extract temporal features in terms of computation time and performance.
\subsubsection{ST-MetaNet}
\textcite{pan2019urban} introduced graph attention network to extract spatial features. They utilize RNN architecture to extract temporal features.

\begin{table}[t!]
\caption{RMSE of different methods\\for PEMS-BAY dataset.}
\label{table:rmse_methods}
\centering
\begin{tabular}{lccc}
\toprule
Horizon &    15 min&    30 min&    60 min\\
\midrule
FC-LSTM~\cite{hochreiter1997long}        &   4.19 &  4.55 &  4.96\\
DCRNN~\cite{li2018diffusion}        &  2.95 &  3.97 &  4.74 \\
STGCN~\cite{yao2018deep}        &  2.96 &  4.27 &  5.69\\
Graph WaveNet~\cite{xu2018graph} &  {\bf{2.74}} &  {\bf{3.70}} &  4.52 \\
ST-MetaNet~\cite{pan2019urban}        &  2.90 &  4.02 &  5.06 \\
Proposed       &  2.90 &  3.77 &  {\bf{4.44}} \\
\bottomrule
\end{tabular}
\end{table}

Table~\ref{table:rmse_methods} shows the RMSE of each model and our proposed method.
We confirm that the performance of the proposed method reaches that of state-of-the-art methods based on a complex deep learning architecture. It even performs better for long-term prediction as we model based on DLM that explicitly expresses the daily periodicity of traffic signals.
For example, the RMSEs of our proposed method for 90 and 120 min horizons are respectively 4.70 and 5.26, while these are 5.26 and 6.02 with the pre-trained DCRNN model.\footnote{As GraphWaveNet predicts all the horizons at once (not recursive), we could not use the pre-trained model for the longer horizons. As a result, we choose DCRNN which shows the second-best result on 60 min horizon.}

Our proposed method requires lower computational effort compared to the others.
Also, it infers the majority of the parameters ($N^2$) analytically by Eq.~(\ref{eqn:probable_H}).
The method only requires numerical computation when it solves the optimization problem~(\ref{eqn:evidence_maximization}) to infer $K+2$ parameters, which has $O(K^2)$ complexity, where $K^2$ is noticeably smaller than $N^2$.
Note that the hyperparameters are optimally estimated by solving the optimization problem~(\ref{eqn:evidence_maximization}) rather than the cross-validation method. As hyperparameter tuning is an expensive task, it can be a major advantage of the proposed method.

On the other hand, all state-of-the-art methods require heavy numerical computations to train a large number of parameters as they are based on deep-neural-net architectures.
Our method successfully infers all parameters at the time scale of minutes with CPU computations, which is noticeably shorter than other DNN based methods with GPU computations as shown in Table~\ref{table:computation} (note that the DNN based methods required from 50 epochs to 100 epochs to converge).

\begin{table}[t!]
\caption{Computation costs for training on the PEMS-BAY dataset}
\label{table:computation}
\centering
\begin{tabular}{lr}
\toprule
Model &    Training(s)\\
\midrule
DCRNN~\cite{li2018diffusion}        &  750 (per epoch)\\
Graph WaveNet~\cite{xu2018graph} &  580 (per epoch)\\
Proposed       &  760 (total)\\
\bottomrule
\end{tabular}
\end{table}

Another advantage of our model compared to the deep-learning-based architectures is that only a small number of parameters need to be decided heuristically. This can provide easy scalability to apply our model to other traffic datasets or datasets with similar properties to traffic data (daily periodicity).
For example, in our model, the parameters to be determined before training are the threshold constant $\kappa$, the kernel width $\sigma$ to build a proper graph, and the number of diffusion processes $K$ to determine how many diffusion processes should be mixed. 
We empirically choose the constants $\kappa$ and $\sigma$ such that the corresponding graph $\mathcal{G}$ is a $k$-vertex-connected graph with a small number $k$.
For the number of diffusion processes $K$, we set $K=5$ for the PEMS-BAY dataset but the prediction performance is not sensitive to the parameter ($\pm0.01$ minutes changes of the RMSE score from $K=3$ to $K=7$).

\section{Conclusion}
In this paper, we proposed a method for predicting traffic signals in transportation sensor networks. 
We successfully integrated topological information of the sensor network into a data-driven model by assuming that the parameters in the model are supported by the mixture of diffusion kernels with uncertainties.
We exploited the Bayesian inference to optimally determine the parameters that characterize the distribution of diffusion processes and the importance of measurements against prior information. 
The importance varies with time, and we discover that the data are relatively more important, especially for the peak period.
Most importantly, the proposed method reached accurate prediction at the level of state-of-the-art methods with less computational effort. 
It particularly shows excellent performance in long-term predictions by exploiting DLM's periodicity modeling.
Our method can be applicable for predicting graph signals exhibiting daily patterns such as weather or energy consumption.
For future works, we may improve the short-term prediction performance if we give more valuable prior information (e.g., graph structure more suitable for prediction; currently, it only depends on topology), or if it is possible to derive all inference processes (especially the marginalization steps in Eq.~(\ref{eqn:evidence-intermeditate}) and (\ref{eqn:evidence_marginalization})) with a non-linear model overcoming the limitation of linear models.

\printbibliography


%

\appendix

\subsection{Volume conservation of mixture of heat diffusion}\label{apdx:volume_conservation_heat_diffusion}
By definition (in Eq.~(\ref{eqn:laplacian})), the graph Laplacian ${\bf{L}}(\mathcal{G})$ has an eigenvector $\frac{1}{\sqrt{N}}{\bf{1}}$ with the corresponding eigenvalue $0$. Let an eigen-decomposition of the matrix be
\begin{equation}
    {\bf{L}}(\mathcal{G})={\bf{V}}{\bf{D}}{\bf{V}}^T,
\end{equation}
where the orthonormal matrix ${\bf{V}}$ and the diagonal matrix ${\bf{D}}$ contain eigenvectors and corresponding eigenvalues, respectively.
Since the orthonormal matrix ${\bf{V}}$ contains the eigenvector $\frac{1}{\sqrt{N}}{\bf{1}}$,
\begin{equation}
\begin{aligned}
{\bf{1}}^T\tilde{\bf{x}}^d_{t+1}(\tau)&\overset{(\ref{eqn:internal_diffusion})}{=}{\bf{1}}^T{\bf{H}}^{\mathcal{G}}(\tau){\bf{x}}^d_{t}\\
&\overset{(\ref{eqn:heat_diffusion_model})}{=}{\bf{1}}^T e^{-\tau{\bf{L}}(\mathcal{G})}{\bf{x}}_t^d={\bf{1}}^T{\bf{V}} e^{-\tau{\bf{D}}}{\bf{V}}^T{\bf{x}}_t^d\\
&=\frac{N}{\sqrt{N}}\frac{1}{\sqrt{N}}{\bf{1}}^T{\bf{x}}_t^d={\bf{1}}^T{\bf{x}}_t^d.
\end{aligned}
\end{equation}
Therefore,
\begin{equation}
    \begin{aligned}
    {\bf{1}}^T\tilde{\bf{x}}^d_{t+1}(\mathcal{T})&={\bf{1}}^T{\bf{H}}^{\mathcal{G}}(\mathcal{T}){\bf{x}}^d_{t}={\bf{1}}^T\left(\sum_\tau \pi^{(\tau)}{\bf{H}}^{\mathcal{G}}(\tau)\right){\bf{x}}^d_{t}\\
    &=\sum_\tau \pi^{(\tau)}{\bf{1}}^T{\bf{H}}^{\mathcal{G}}(\tau){\bf{x}}^d_{t}=\sum_\tau \pi^{(\tau)}{\bf{1}}^T{\bf{x}}^d_{t}\\
    &={\bf{1}}^T{\bf{x}}^d_{t}\sum_\tau \pi^{(\tau)}={\bf{1}}^T{\bf{x}}^d_{t}.
    \end{aligned}
\end{equation}

\subsection{Evidence}\label{apdx:evidence}
\begin{equation}\label{eqn:evidence_marginalization}
    \begin{aligned}
&f({\bf{X}}_{t+1}|{\bf{X}}_t,\alpha_t,{\Pi}_t)\\
&=\int f({\bf{X}}_{t+1}|{\bf{X}}_t,{\bf{H}}_t,\alpha_t) f({\bf{H}}_t|{\Pi}_t)d{\bf{H}}_t
\\&\propto\int e^{-\frac{1}{2}\alpha_t\text{tr}\{({\bf{X}}_{t+1}-{\bf{H}}_t{\bf{X}}_t)({\bf{X}}_{t+1}-{\bf{H}}_t{\bf{X}}_t)^T\}}\\
&\qquad\cdot e^{-\frac{1}{2}\gamma_t\text{tr}\{({\bf{H}}_{t}-{\bf{H}}_t^{\mathcal{G}}(\mathcal{T}))({\bf{H}}_{t}-{\bf{H}}_t^{\mathcal{G}}(\mathcal{T}))^T\}}d{\bf{H}}_t
\\&\propto e^{-\frac{1}{2}\alpha_t\left({\bf{X}}_{t+1}({\bf{I}}-\alpha_t{\bf{X}}_t^T{\boldsymbol{\Sigma}}_t{\bf{X}}_t){\bf{X}}_{t+1}^T-2\gamma_t{\bf{H}}_t^{\mathcal{G}}(\mathcal{T}){\boldsymbol{\Sigma}}_t{\bf{X}}_t{\bf{X}}_{t+1}^T\right)}
\\&\qquad\cdot\int (2\pi)^{-\frac{N^2}{2}}|{\boldsymbol{\Sigma}}_t|^{-\frac{N}{2}}e^{-\frac{1}{2}\text{tr}\{({\bf{H}}_t-\hat {\bf{H}}_t){\boldsymbol{\Sigma}}_t^{-1}({\bf{H}}_t-\hat {\bf{H}}_t)^T\}}d{\bf{H}}_t
\\&\propto e^{-\frac{1}{2}\alpha_t\left({\bf{X}}_{t+1}({\bf{I}}-\alpha_t{\bf{X}}_t^T{\boldsymbol{\Sigma}}_t{\bf{X}}_t){\bf{X}}_{t+1}^T-2\gamma_t{\bf{H}}_t^{\mathcal{G}}(\mathcal{T}){\boldsymbol{\Sigma}}_t{\bf{X}}_t{\bf{X}}_{t+1}^T\right)}
\\&\propto e^{-\frac{1}{2}\text{tr}\{\alpha_t({\bf{X}}_{t+1}-{\bf{H}}_t^{\mathcal{G}}(\mathcal{T}){\bf{X}}_t)({\bf{I}}+\alpha_t\gamma_t^{-1}{\bf{X}}_t^T{\bf{X}}_t)^{-1}({\bf{X}}_{t+1}-{\bf{H}}_t^{\mathcal{G}}(\mathcal{T}){\bf{X}}_t)^T\}},
\end{aligned}
\end{equation}
where ${\boldsymbol{\Sigma}}_t^{-1}=\alpha_t{\bf{X}}_t{\bf{X}}_t^T+\gamma_t{\bf{I}}$. 

\subsection{Posterior of ${\bf{x}}_{t+h}$}\label{apdx:prediction}
When $h=1$,
\begin{equation}\label{eqn:evidence}
\begin{aligned}
&f({\bf{x}}_{t+1}|{\bf{x}}_t,{\bf{X}}_{t+1},{\bf{X}}_t)
\\&=\int f({\bf{x}}_{t+1}|{\bf{x}}_t,{\bf{H}}_t,\alpha_t)f({\bf{H}}_t|{\bf{X}}_{t+1},{\bf{X}}_t,\alpha_t,\gamma_t,\Pi_t,\mathcal{G})d{H}_t
\\&=f({\bf{x}}_{t+1}|{\bf{x}}_t,\hat {\bf{H}}_t,\alpha_t)=\mathcal{N}(\hat {\bf{H}}_t{\bf{x}}_t,\alpha_t^{-1}{\bf{I}}).
\end{aligned}
\end{equation}
Assume the statement is true for $h=l-1$ so that
\begin{equation}\label{eqn:km1assumption}
\begin{aligned}
f({\bf{x}}_{t+l-1}|{\bf{x}}_t,{\bf{X}}_{t:t+l-1})=\mathcal{N}(\hat {\bf{H}}_{t+l-2\leftarrow t}{\bf{x}}_t,{\bf{R}}_{t+l-2}),
\end{aligned}
\end{equation}
where $\hat {\bf{H}}_{t+l-2\leftarrow t}=\hat {\bf{H}}_{t+l-2}\hat {\bf{H}}_{t+l-3}\cdots \hat {\bf{H}}_{t}.$ By the chain rule,
\begin{equation}
    \begin{aligned}
&f({\bf{x}}_{t+l}|{\bf{x}}_t,{\bf{X}}_{t:t+l})\\
&=\int f({\bf{x}}_{t+l}|{\bf{x}}_{t+l-1},{\bf{X}}_{t+l-1}) f({\bf{x}}_{t+l-1}|{\bf{x}}_t,{\bf{X}}_{t:t+l-1})d{\bf{x}}_{t+l-1}.
\end{aligned}
\end{equation}
Since
\begin{equation}
    \begin{aligned}&f({\bf{x}}_{t+l}|{\bf{x}}_{t+l-1},{\bf{X}}_{t+l},{\bf{X}}_{t+l-1})
    \\&\qquad\qquad\qquad\qquad\qquad\overset{(\ref{eqn:evidence})}{=}\mathcal{N}(\hat {\bf{H}}_{t+l-1}{\bf{x}}_{t+l-1},\alpha_{t+l-1}^{-1}{\bf{I}}),\\
&f({\bf{x}}_{t+l-1}|{\bf{x}}_t,{\bf{X}}_{t+l-1},\cdots,{\bf{X}}_t)\overset{(\ref{eqn:km1assumption})}{=}\mathcal{N}(\hat {\bf{H}}_{t+l-2\leftarrow t}{\bf{x}}_t,{\bf{R}}_{t+l-2}),
\end{aligned}
\end{equation}

\begin{equation}\label{eqn:evidence-intermeditate}
    \begin{aligned}&f({\bf{x}}_{t+l}|{\bf{x}}_t,{\bf{X}}_{t+l},\cdots,{\bf{X}}_t)\\
&=\int \mathcal{N}(\hat {\bf{H}}_{t+l-1}{\bf{x}}_{t+l-1},\alpha_{t+l-1}^{-1}{\bf{I}})
\\&\qquad\qquad\qquad\qquad\cdot\mathcal{N}(\hat {\bf{H}}_{t+l-2\leftarrow t}{\bf{x}}_t,{\bf{R}}_{t+l-2})d{\bf{x}}_{t+l-1}
\\&\propto \int 
exp\Big(-{\frac{1}{2}} \big\{\alpha_{t+l-1}({\bf{x}}_{t+l}-
\hat {\bf{H}}_{t+l-1}{\bf{x}}_{t+l-1})^T
\\&\qquad\qquad\qquad\qquad\cdot({\bf{x}}_{t+l}-\hat {\bf{H}}_{t+l-1}{\bf{x}}_{t+l-1})
\\&+({\bf{x}}_{t+l-1}-\hat {\bf{H}}_{t+l-2\leftarrow t}{\bf{x}}_{t})^T{\bf{R}}_{t+l-2}
\\&\qquad\qquad\qquad\qquad\cdot({\bf{x}}_{t+l-1}-\hat {\bf{H}}_{t+l-2\leftarrow t}{\bf{x}}_{t})\big\}\Big)d{\bf{x}}_{t+l-1}
\\&\propto exp\Big(-\frac{1}{2}\big(\alpha_{t+l-1}{\bf{x}}_{t+l}^T{\bf{x}}_{t+l}
\\&\qquad\qquad-(\alpha_{t+l-1}\hat {\bf{H}}_{t+l-1}^T{\bf{x}}_{t+l}+{\bf{R}}_{t+l-2}^{-1}\hat {\bf{H}}_{t+l-2\leftarrow t}{\bf{x}}_t)^T
\\&\qquad\qquad\cdot(\alpha_{t+l-1}\hat {\bf{H}}_{t+l-1}^T\hat {\bf{H}}_{t+l-1}+{\bf{R}}_{t+l-2}^{-1})^{-1}
\\&\qquad\qquad\cdot( \alpha_{t+l-1}\hat {\bf{H}}_{t+l-1}^T{\bf{x}}_{t+l}+{\bf{R}}_{t+l-2}^{-1}\hat {\bf{H}}_{t+l-2\leftarrow t}{\bf{x}}_t)\big)\Big)
\\&\propto exp\Big(-\frac{1}{2}\alpha_{t+l-1}
\\&\cdot\big({\bf{x}}_{t+l}^T({\bf{I}}-\alpha_{t+l-1}\hat {\bf{H}}_{t+l-1}(\alpha_{t+l-1}\hat {\bf{H}}_{t+l-1}^T\hat {\bf{H}}_{t+l-1}
\\&\qquad\qquad \qquad\qquad\qquad\qquad+{\bf{R}}_{t+l-2}^{-1})^{-1}\hat {\bf{H}}_{t+l-1}^T){\bf{x}}_{t+l}
\\&-2{\bf{x}}_{t+l}^T\hat {\bf{H}}_{t+l-1}(\alpha_{t+l-1}\hat {\bf{H}}_{t+l-1}^T\hat {\bf{H}}_{t+l-1}+{\bf{R}}_{t+l-2}^{-1})^{-1}\\&\qquad\qquad\qquad\qquad \qquad\qquad\qquad\cdot {\bf{R}}_{t+l-2}^{-1}\hat {\bf{H}}_{t+l-2\leftarrow t}{\bf{x}}_t\big)\Big).
\end{aligned}
\end{equation}
Applying matrix inversion lemma, Eq.~(\ref{eqn:evidence-intermeditate}) becomes
\begin{equation}
\begin{aligned}
&exp\Big(-\frac{1}{2}\alpha_{t+l-1}\\&\cdot\big({\bf{x}}_{t+l}^T({\bf{I}}+\alpha_{t+l-1}\hat {\bf{H}}_{t+l-1}{\bf{R}}_{t+l-2}\hat {\bf{H}}_{t+l-1}^T)^{-1}{\bf{x}}_{t+l}
\\&-2{\bf{x}}_{t+l}^T({\bf{I}}+\alpha_{t+l-1}\hat {\bf{H}}_{t+l-1}{\bf{R}}_{t+l-2}\hat {\bf{H}}_{t+l-1}^T)^{-1}
\\&\qquad\qquad\qquad\qquad\qquad\qquad\qquad\cdot\hat {\bf{H}}_{t+l-1}\hat {\bf{H}}_{t+l-2\leftarrow t}{\bf{x}}_t\big)\Big)
\\&\propto exp\Big(-\frac{1}{2}({\bf{x}}_{t+l}-\hat {\bf{H}}_{t+l-1}\hat {\bf{H}}_{t+l-2\leftarrow t}{\bf{x}}_t)^T
\\&\quad\qquad\qquad\qquad\cdot {\bf{R}}_{t+l-1}^{-1}({\bf{x}}_{t+l}-\hat {\bf{H}}_{t+l-1}\hat {\bf{H}}_{t+l-2\leftarrow t}{\bf{x}}_t)\Big),
\end{aligned}
\end{equation}
where ${\bf{R}}_{t+l-1}=\alpha_{t+l-1}^{-1}{\bf{I}}+\hat {\bf{H}}_{t+l-1}{\bf{R}}_{t+l-2}\hat {\bf{H}}_{t+l-1}^T$
and by definition $\hat {\bf{H}}_{t+l-1\leftarrow t}=\hat {\bf{H}}_{t+l-1}\hat {\bf{H}}_{t+l-2\leftarrow t}$, so 
\begin{equation}
 f({\bf{x}}_{t+l}|{\bf{x}}_t,{\bf{X}}_{t+l},\cdots,{\bf{x}}_t)=\mathcal{N}(\hat {\bf{H}}_{t+l-1\leftarrow t}{\bf{x}}_t,{\bf{R}}_{t+l-1}).
\end{equation}
Finally ${\bf{x}}_{t+h|t}=\hat {\bf{H}}_{t+h-1}\cdots\hat {\bf{H}}_t{\bf{x}}_t$.



\ifCLASSOPTIONcaptionsoff
  \newpage
\fi

\end{document}